\documentclass[a4paper]{article}
\usepackage{amssymb}
\usepackage{amsmath}
\usepackage[hyphens]{url} 
\usepackage[utf8]{inputenc}
\usepackage{mathtools}  
\usepackage{amsthm}
\usepackage{color}
\usepackage{nicefrac}
\usepackage{tcolorbox}
\usepackage{natbib}
\usepackage{subcaption}
\usepackage{graphicx}
\usepackage{float}
\usepackage{dsfont}
 \newtheorem{thm}{Theorem}[section]
 \newtheorem{algo}{Algorithm}
 \newtheorem{prop}{Proposition}[section]
 \newtheorem{lem}{Lemma}[section]
 \newtheorem{cor}{Corollary}[section]
 \newtheorem{ass}{Assumption}
 \newtheorem{exm}{Example}[section]
 \newtheorem{dfn}{Definition}
 \newtheorem{rem}{Remark}

\def\la{\langle}
\def\ra{\rangle}
\def\R{\mathbb{R}}
\def\E{\mathbb{E}}
\def\P{\mathbb{P}}
\def\Rn{\mathbb{R}^n}

\def\R{\mathbb{R}}
\title{Beyond Softmax: A New Perspective on Gradient Bandits\thanks{This article presents a revised and expanded treatment of \emph{Discrete Choice Multi-Armed Bandits}, superseding the earlier  preliminary report.}}

\author{Emerson Melo \thanks{Department of Economics,  Indiana University Bloomington; e-mail:  emelo@iu.edu}  \and 
David Müller \thanks{e-mail: dgfin@gmx.de}  } 
\providecommand{\keywords}[1]{\textbf{\textit{Keywords---}} #1}

\begin{document}

\maketitle
\vspace{18mm} \setcounter{page}{1}

%\begin{abstract}	
%This paper establishes a connection between a class of discrete choice models and the fields of online learning and multiarmed bandit algorithms. Our contributions are twofold. First, we provide sublinear regret bounds for a broad family of algorithms, with the Exp3 algorithm emerging as a special case. Second, we introduce a new class of adversarial multiarmed bandit algorithms inspired by the generalized nested logit models of \citet{wen:2001}. These algorithms offer significant flexibility in model specification and are computationally efficient due to their closed-form sampling probabilities. To illustrate their practical applicability, we present numerical experiments in the stochastic bandit setting.\end{abstract}
\begin{abstract}
    We establish a link between a class of discrete choice models and the theory of online learning and multi-armed bandits. Our contributions are: (i) sublinear regret bounds for a broad algorithmic family, encompassing Exp3 as a special case; (ii) a new class of adversarial bandit algorithms derived from generalized nested logit models \citep{wen:2001}; and  (iii) \textcolor{black}{we introduce a novel class of generalized gradient bandit algorithms that extends beyond the widely used softmax formulation. By relaxing the restrictive independence assumptions inherent in softmax, our framework accommodates correlated learning dynamics across actions, thereby broadening the applicability of gradient bandit methods.} Overall, the proposed algorithms combine flexible model specification with computational efficiency via closed-form sampling probabilities. Numerical experiments in stochastic bandit settings demonstrate their practical effectiveness.
\end{abstract}
\section{Introduction}\label{s1}
The multi-armed bandit (MAB) problem is a foundational framework in decision theory and reinforcement learning, formalizing the trade-off between exploration and exploitation in uncertain environments. Inspired by the analogy of a gambler choosing among multiple slot machines, MAB models have been widely applied in economics, operations research, and computer science.

In economics, MAB models play a central role in dynamic decision-making. A prominent example is dynamic pricing, where firms adapt prices in real time to respond to demand fluctuations and maximize revenue. In online retail, for instance, \cite{besbes2009dynamic} show that bandit algorithms perform effectively under demand uncertainty. Similarly, in finance, MAB models inform portfolio management by balancing the exploration of new investment opportunities with the exploitation of established assets \citep{bertsimas2007dynamic}. They are also critical in auction design and online advertising, where advertisers employ bandit algorithms to allocate budgets across ad placements, improving bidding strategies and campaign performance \citep{gatti2012exploration}.

The study of MAB problems originates with the seminal work of \citet{robbins1952some}, which first formalized the exploration–exploitation trade-off. Since then, algorithms such as $\varepsilon$-greedy, Upper Confidence Bound (UCB), and Thompson Sampling have been extensively studied and refined. For comprehensive surveys of these classical approaches, see \citet{lattimore} and \citet{slivkins}.

\smallskip

In this paper, we advance the study of MAB by systematically leveraging tools from discrete choice theory to inform the design of online optimization and bandit algorithms. Our contributions are threefold:

First, in the experts setting, we revisit the connection between the Gradient-Based Prediction Algorithm (GBPA) and the surplus function of Random Utility Models (RUM), as recently established by \citet{melo2021learning} in the complete-feedback case. Building on this foundation, we incorporate components from discrete choice models \citep{prox} into the GBPA framework of \citet{abernethy2016perturbation}, yielding sublinear regret guarantees for a broad class of algorithms. The well-known Exp3 algorithm emerges as a special case, and our reanalysis with Gumbel smoothing yields improved regret bounds. We also compare algorithmic behavior across surplus functions derived from the traditional softmax Multinomial Logit (MNL) and the more flexible Nested Logit (NL) models.

Second, we study the adversarial MAB problem and derive sublinear expected regret bounds for algorithms in the loss-only framework of \citet{abernethy2016perturbation}. Our main contribution here is the introduction of a new family of adversarial MAB algorithms based on Generalized Extreme Value (GEV) models. We identify a sufficient condition for differential consistency in GEV-based algorithms and show that Generalized Nested Logit (GNL) models \citep{wen:2001} satisfy this condition. Since GNL subsumes the widely used MNL model underlying Exp3 \citep{cesa2006prediction}, our framework provides a unifying view of GBPA bandit algorithms. Crucially, this extension accommodates correlations across arms and complex nesting structures, substantially broadening the scope of adversarial bandit methods beyond what is addressed in prior work.

Third, we address the stochastic MAB problem by introducing the Generalized Gradient Bandit Algorithms, a family of methods grounded in GNL models. This family strictly generalizes the classical Gradient Bandit Algorithm \citep{sutton}, extending its applicability by incorporating richer preference structures. To the best of our knowledge, prior analyses of gradient bandit methods have been limited to the MNL model. Our approach, by contrast, exploits correlations across arms to refine exploration–exploitation trade-offs, even in flat (non-hierarchical) settings. This allows the algorithm to prioritize promising subsets of actions and capture dependencies that standard gradient bandit methods cannot.

A distinctive feature of our framework is that the nested structure influences both the sampling distribution and the preference-update dynamics. Observed rewards are partially propagated to related arms, enabling information sharing within groups of alternatives. This property is advantageous in a range of applications, including recommendation systems, dynamic pricing, healthcare, advertising, and reinforcement learning. Importantly, the family of Generalized Gradient Bandit Algorithms retains closed-form sampling probabilities, ensuring computational efficiency.

To evaluate effectiveness, we conduct simulation experiments with an NL-based variant of our algorithm. This NL bandit recovers the classical Gradient Bandit Algorithm as a special case, guaranteeing equivalent performance in unstructured environments. Our experiments show that, when structural assumptions are present, NL bandit variants consistently outperform the baseline. The observed improvements stem from more informed exploration: the nested structure enables information sharing across related arms, accelerating the identification and exploitation of high-reward options.

\subsection{Related literature} The literature on MAB is extense and for an excellent textbook treatment for classical models and methods we refer the reader to \cite{lattimore} and \cite{slivkins}.  In the context of adversarial multi-armed bandits (MAB), \citet{abernethy2016perturbation} propose a unifying framework for the Follow-The-Regularized-Leader (FTRL) and Follow-The-Perturbed-Leader (FTPL) algorithms, grounded in gradient-based potential functions. Specifically, they introduce the Gradient-Based Prediction Algorithm (GBPA), which leverages convex potential functions to update predictions via their gradients, thereby providing a unified perspective that encompasses both FTRL and FTPL methods. Importantly, they demonstrate that GBPA can be implemented in the adversarial MAB setting and establish sublinear regret guarantees under the condition of differential consistency. Their analysis applies to the case of independent arms.

Our work departs from \citet{abernethy2016perturbation} in at least three key respects. First, we establish a novel connection between GBPA and the GEV family, which enables us to provide an explicit characterization of differential consistency. Second, we extend sublinear regret guarantees to settings where actions may be correlated and decision-making exhibits a nested structure—scenarios not addressed in their work. Finally, we introduce and analyze generalized gradient bandit algorithms, a class of methods absent from \citet{abernethy2016perturbation}’s discussion.
\smallskip

Recent work has extended the MAB framework to incorporate richer modeling assumptions, including contextual information \citep{li2010contextual} and hierarchical structures \citep{martin2022nested}. These extensions enhance decision-making by leveraging additional sources of information beyond simple reward signals. Notable advances include the development of UCB algorithms \citep{auer} and the application of Bayesian methods \citep{agrawal2012thompson}, which have enabled more adaptive and statistically principled bandit models. More recently, \citet{li2024optimism} relaxed the traditional FTPL assumption of independently and identically distributed (i.i.d.) noise across arms by allowing correlated perturbations. Yet, these studies do not address the GEV class or the generalized gradient bandit algorithm. By leveraging the GEV framework, we are able to derive explicit nesting structures, which provide additional insight.
\smallskip

Similarly, \citet{lee2025revFTPL} exploit the theory of RUM to construct hybrid choice models that yield Best-of-Both-Worlds guarantees. However, their analysis does not consider the GEV class or the generalized gradient bandit algorithm.
\smallskip

Finally, in the context of generalized gradient algorithms, the closest work to ours is \citet{martin2022nested}. Their approach is hierarchical in nature, embedding the NL model into a multi-level decision structure to capture complex inter-arm relationships. In contrast, while hierarchical models rely on explicit nesting across decision layers, our approach operates within the standard flat MAB paradigm and incorporates nested preferences directly into the choice model.  \\

\noindent {\bfseries Notation:}  Our notation is quite standard. By $\mathbb{R}^n $ we denote the space of n-dimensional vectors, where  the vectors $x = \left(x^{(1)}, x^{(2)}, \ldots, x^{(n)}\right)^T $ are column vectors. For $x \in \Rn$ we write $x^{-(i)} \in \R^{n-1}$ meaning that the $i$-th component of $x$ is missing. Analogously, we write $x^{-(i,j)} \in \R^{n-2}$ which means that the components $i$-th and $j$-th  of $x$ are missing. Using the latter, we write with some abuse of notation:
 \[
 x = \left(x^{-(i,j)}, x^{(j)}, x^{(i)}
 \right).
 \] 

 We denote by $e_j \in \R^n$ the $j$-th coordinate vector of $\R^n$ and write  $e$ for the vector of an appropriate dimension whose components are equal to one. Similarly, we write $\mathbf{0}$ for the vector of an appropriate dimension whose components are equal to zero. For a vector $x \in \R^n$ we write $e^x$ for the exponential operation of all the components, i.\,e. 
 \[
 e^x = \left(e^{x{(1)}}, e^{x{(2)}}, \ldots, e^{x{(i)}}, \ldots, e^{x{(n)}} \right)^T.
 \]
 With this convention, the following holds:
 \[
 e^{\mathbf{0}} = \left(e^{0}, \ldots, e^{0}, \ldots, e^{0} \right)^T = \left(1, \ldots, 1, \ldots, 1 \right)^T = e^T.
 \]
 By $\R^n_+$ we denote the set of all vectors with nonnegative components.
We introduce the standard inner product in $\mathbb{R}^n$:
\[
\left\langle x,y\right\rangle = \sum\limits_{i=1}^{n} x^{(i)} y^{(i)}. 
\]
 If $y > 0$ we define the vector division:
 \[
 \frac{x}{y} = \left(\frac{x^{(1)}}{y^{(1)}}, \ldots, \frac{x^{(n)}}{y^{(n)}}\right)^T.
 \]
 For $x \in \mathbb{R}^n $ we use the norms
\[
\|x\|_1 = \sum\limits_{i=1}^{n} |x^{(i)}|, \quad 
\|x\|_2 = \sqrt{\sum\limits_{i=1}^{n} \left(x^{(i)}\right)^2}, \quad
\|x\|_\infty = \underset{1 \le i \le n}{\max} |x^{(i)}| . 
\] 
Given a function $ f $ we denote its  domain by  $\mbox{dom} f = \{x \in \mathbb{R}^n \, | \,f(x) < \infty\}. $ Further, we recall the definition of the convex conjugate of the function $f: $ \[f^\star(x^\star) = \underset{x \in \mathbb{R}^n}{\sup}\left\langle x,x^\star \right\rangle - f(x), \] where $x^\star $ is a vector of dual variables. Finally, for the $(n-1)$-dimensional simplex we write   \[
\Delta_n = \left\{p \in \mathbb{R}^n \,\left|\, \sum\limits_{i=1}^{n} p^{(i)} = 1, \;  p^{(i)} \ge 0, \; i=1, \ldots, n \right. \right\}. 
\]
The Bregman divergence of a convex function $f$ is given by:
\[
D_f(y,x) = f(y) - f(x) - \la \nabla f(x), y-x \ra, \quad \mbox{for all} \quad x,y \in \mbox{dom} f.
\]
A function $f: \R^n \rightarrow \R$ is $L$-strongly smooth w.r.t. $\|\cdot\|$ norm if it is differentiable, and for all $x, y \in \R^n$ we have:
\[
  f(y) \leq f(x) + \langle \nabla f(x), y-x \rangle + \frac{L}{2} \|y-x\|^2.
\]
The positive constant $L$ is called the smoothness parameter of $f$.
Obviously, for a $L$-strongly smooth function $f$ it holds:
\[
D_f(y,x) \leq  \frac{L}{2} \|y-x\|^2.
\]

\section{Discrete Choice Review}\label{s2}  %% Please avoid complex formulas in (sub)titles
In this section, we review discrete choice behavior as modeled by additive random utility models (ARUM). We argue that these models provide a natural foundation for designing online optimization algorithms in the experts setting. Additionally, we summarize recent results on the algorithmic aspects of discrete choice models, with a particular focus on the connection between ARUMs and online optimization. In online optimization, data becomes available sequentially rather than in batches. At each iteration, a new data point arrives, prompting an update to the decision. Unlike in MABs, however, the agent observes the full vector of payoffs or losses, enabling more informed updates.
\subsection{The ARUM class}\label{ss:ARUM}
 The ARUM class—characterized by the additive decomposition of utility into deterministic and stochastic components—was first introduced in the seminal work of \citet{thurstone}, which aimed to rationalize stimulus-response experiments. A formal description of this framework was later introduced in the economics literature by \citet{gev}, where rational decision-makers are assumed to choose from a finite set of mutually exclusive alternatives $A = \{1, \ldots, n\}$. Each alternative $i \in A$ yields the random utility $\tilde{u}^{(i)}$, which is assumed to have the following additive structure:
 \[\tilde{u}^{(i)}=u^{(i)} + \epsilon^{(i)}, \] where $ u^{(i)} \in \mathbb{R} $ is the deterministic utility part of the $i$-th alternative and $\epsilon^{(i)} $ is its stochastic error component.  In the economic literature, the term is often $\epsilon^{(i)} $ referred as alternative $i$'s preference shocks which captures the idea of unobserved heterogeneity across a population of decision-makers (\cite{mcfadden}).
 
For the sake of clarity, throughout the paper we use vector notation to represent both the deterministic and random utilities, respectively:
\[
u = \left(u^{(1)}, \ldots, u^{(n)}\right)^T, \quad \epsilon = \left(\epsilon^{(1)}, \ldots, \epsilon^{(n)}\right)^T. 
\] 

Next, we introduce the notion of the surplus function, which plays a central role in the ARUM framework and, consequently, in our analysis. Formally, the surplus function is defined as the expected maximum overall utility:
\begin{equation}\label{consumer surplus general}
E(u) \triangleq
\mathbb{E}_\epsilon  \left[\max_{i\in A} \{u^{(i)} + \epsilon^{(i)}\}\right].   
\end{equation}

The following assumption concerning random errors is standard, see e.g. \citet{palma}. 
\begin{ass}[]\label{ass:joint density}
	The random vector $\epsilon $ follows a joint distribution with zero mean that is absolutely continuous with respect to the Lebesgue measure fully supported on $\mathbb{R}^n. $
\end{ass} 
Under Assumption \ref{ass:joint density}, the surplus function is convex and differentiable \citep{palma}. In particular, the well-known Williams-Daly-Zachary theorem  states that the gradient of $E$ corresponds to the vector of choice probabilities \citep{gev} which can be stated in terms of partial derivatives of $E$:  \begin{equation}\label{eq:dalytheorem}
\frac{\partial E(u)}{\partial u^{(i)}} = \mathbb{P} \left( u^{(i)} +  \epsilon^{(i)} = \underset{ l\in A}{\max} \{u^{(l)} + \epsilon^{(l)} \}\right), \quad\forall i \in A.
\end{equation}

We denote this probability by $\mathbb{P}^{(i)}$. 
This formula holds due to Assumption \ref{ass:joint density} as  ties in Equation \eqref{consumer surplus general} occur with probability zero. Furthermore, we note that a particular distribution for $\epsilon$, expressions \eqref{consumer surplus general} and \eqref{eq:dalytheorem}  fully characterize the surplus 
function $E$ and the choice probability vector $\mathbb{P}$. 

\subsection{The Generalized Extreme Value model }We now focus in the class of Generalized Extreme Value (GEV) models, introduced by \citet{mcfadden1978modelling} and \citet{mcfadden}. The GEV class encompasses a broad range of models, including the widely used MNL and NL models. 

In developing the GEV class, McFadden introduces the notion of a \emph{generator function}, which we define now.
\begin{dfn}\label{Generator_GEV} A  function $G:\Rn_{+}\longrightarrow\R_+$ is a generator if the following conditions hold:
 	\begin{itemize}
 		\item[(i)]Non-negativity: For all $x=(x^{(1)},\ldots,x^{(n)})\in \Rn_+$, $G(x)\geq 0$.
 		\item[(ii)]Homogeneity of degree 1:  $G(\lambda x)=\lambda G(x)$ for all $x\in \Rn_+$ and $\lambda>0.$
 		\item[(iii)]Coercivity in each argument: For each $i = 1, \ldots, n$, it holds that $G(x) \to \infty$ as $x^{(i)} \to \infty$, with all other components of $x$ held fixed.
 		\item[(iv)]   Sign structure of cross-partial derivatives: For any set of $k$ distinct indices $i_1, \ldots, i_k \in \{1, \ldots, n\}$, the $k$-th order mixed partial derivative satisfies:
$$\frac{\partial^k G(x)}{\partial x^{(i_1)} \cdots \partial x^{(j_k})} \begin{cases}
\geq 0 & \text{if } k \text{ is odd}, \\
\leq 0 & \text{if } k \text{ is even}.
\end{cases}$$
 	\end{itemize}
\end{dfn}

 \citet{mcfadden1978modelling,mcfadden} show that when the generator function $G$ satisfies conditions (i)–(iv), the random vector $\epsilon = (\epsilon^{(1)}, \ldots, \epsilon_{(n)}$ if it follows the joint distribution given by the following probability density function:
function
\[
f_\epsilon\left(x^{(1)}, \ldots, x^{(n)}\right) = \frac{\partial^n \exp\left(-G\left(e^{-x^{(1)}},\ldots, e^{-x^{(n)}}\right)\right)}{\partial x^{(1)} \cdots \partial x^{(n)}},
\]

It is well-known from \citet{mcfadden1978modelling} that the surplus function for GEV is
\begin{equation}\label{eq:gev_surplus}
E(u) = \mu \ln G\left(e^{u}\right),
\end{equation}
where we neglect an additive constant.

Using the result in Eq. \eqref{eq:dalytheorem}, it follows  that the choice probability of the $i$-th alternative  corresponds to:

\begin{equation}\label{eq:gev_choiceprob}
\P^{(i)} =\frac{\partial E(u)}{\partial u^{(i)}} = \mu \frac{\partial G\left(e^{u}\right)}{\partial x^{(i)}}\cdot \frac{e^{u^{(i)}}}{G\left(e^{u}\right)}.
\end{equation}

An important instance of the GEV family, corresponds to the  generalized nested logit (GNL) model  introduced by \citet{wen:2001}, where the    generating function $G$ is defined as follows:

\begin{equation}\label{eq:gnl_generatingfct}
G(x)= \sum_{\ell \in L} \left( \sum_{i =1}^{n} \left(\sigma_{i\ell}\cdot x^{(i)}\right)^{\nicefrac{1}{\mu_\ell}} \right)^{\nicefrac{\mu_\ell}{\mu}}.
\end{equation}
Here, $L$ is a generic set of nests. The parameters $\sigma_{i\ell} \geq 0$ denote the shares of the $i$-th alternative with which it is attached to the $\ell$-th nest. For any fixed $i \in \{1, \ldots,n\}$ they sum up to one:
\[
\sum_{\ell \in L} \sigma_{i\ell}  = 1.
\]
$\sigma_{i\ell}=0$ means that the $\ell$-th nest does not contain the $i$-th alternative. Hence, the set of alternatives within the $\ell$-th nest is
\[
N_\ell = \left\{i \,|\, \sigma_{i\ell} >0\right\}.
\]
The nest parameters $\mu_\ell > 0$ describe the variance of the random errors while choosing alternatives within the $\ell$-th nest. Analogously, $\mu >0$ describes the variance of the random errors while choosing among the nests. For the function $G$ to fulfill (G1)-(G3) we require:
\[
\mu_\ell \leq \mu \quad \mbox{for all } \ell \in L. 
\]
The underlying choice process can be divided into two stages:
\begin{itemize}
	\item[(1)] the probability of choosing the $\ell$-th nest is
	\[
	\hat{\P}^{(\ell )}= \frac{e^{\frac{v_\ell}{\mu}}}{\displaystyle
		\sum_{\ell \in L}e^{\frac{v_\ell}{\mu}}},
	\]
	where
	\[
	v_\ell = \mu_\ell \ln \left( \sum_{i =1}^{n} \left(\sigma_{i\ell} \cdot e^{u^{(i)}}\right)^{\frac{1}{\mu_\ell}} \right)
	\]
	stands for the utility attached to the $\ell$-th nest;
	\item[(2)] the probability of choosing the $i$-th alternative within the $\ell$-th nest is
	\[
	\P^{(i|\ell)} = \frac{\left(\sigma_{i\ell} \cdot e^{u^{(i)}}\right)^{\frac{1}{\mu_\ell}}}{\displaystyle
		\sum_{i=1}^{n} \left(\sigma_{i\ell} \cdot e^{u^{(i)}}\right)^{\frac{1}{\mu_\ell}}}.
	\]
\end{itemize}
Overall, the choice probability of the $i$-th alternative according to GNL amounts to
\[
\P^{(i)} = \mu \frac{\partial G\left(e^{u}\right)}{\partial x^{(i)}}\cdot\frac{e^{u^{(i)}}}{G\left(e^{u}\right)}= \sum_{\ell \in L} \hat{\P}^{(\ell )} \cdot \P^{(i|\ell)} .
\] 

We illustrate the concept of the generating function based on the MNL. Recall that in the MNL  the random errors in \eqref{consumer surplus general} are assumed to be i.i.d Gumbel-distributed. 
\begin{exm}[MNL Model] 
\label{ex:mnl}
The generating function  
\[
G(x)= \sum_{i =1}^{n} \left(x^{(i)}\right)^{\nicefrac{1}{\mu}}
\]
leads to the MNL, since The corresponding surplus function becomes
\[
E(u) = \mu \ln \sum_{i =1}^{n} e^{\nicefrac{u^{(i)}}{\mu}},
\]
and the choice probabilities are
\begin{equation}\label{eq}
\P \left( u^{(i)} + \epsilon^{(i)} = \max_{l\in A} \{u^{(l)} + \epsilon^{(l)}\}\right)= \frac{e^{\nicefrac{u^{(i)}}{\mu}}}{\displaystyle
	\sum_{i=1}^{n}e^{\nicefrac{u^{(i)}}{\mu}}}, \quad i\in A.
\end{equation}
\end{exm}
The MNL model is very popular. However, it is not able to capture non-independent substitution patterns due to the Independence of Irrelevant Alternatives Axiom (IIA). This might be a drawback in designing an online optimization algorithm for several scenarios. Another well-known instance of the GNL family which violates the IIA and is thus be able to deal with dependent alternatives  is the NL model.
\begin{exm}[NL Model]
	\label{ex:nl}
	Let in GNL for every alternative $i \in A= \{1,\ldots,n\}$ there be a unique nest $\ell_i \in L$ with $\sigma_{i \ell_i}=1$, and $\mu=1$. Then, 
		the nests $N_\ell=\left\{i \,|\, \ell_i=\ell\right\}$ are mutually exclusive, and 
		the generating function
		\[
		G(x)=\sum_{\ell \in L} \left( \sum_{i \in N_\ell} x^{(i)\nicefrac{1}{\mu_\ell}} \right)^{\mu_\ell}
		\]
		leads to the nested logit (NL). The corresponding surplus function is
		\[
		E(u) = \mu \ln \sum_{\ell \in L} \left( \sum_{i \in N_\ell} e^{\nicefrac{u^{(i)}}{\mu_\ell}} \right)^{\mu_\ell},
		\]
		and the choice probabilities for $i \in N_\ell$, $\ell \in L$ are
		\[
		\P^{(i)} = 
		\frac{e^{\mu_\ell \ln \sum_{i \in N_\ell} e^{\nicefrac{u^{(i)}}{\mu_\ell}}}}{\displaystyle
			\sum_{\ell \in L} e^{\mu_\ell \ln \sum_{i \in N_\ell} e^{\nicefrac{u^{(i)}}{\mu_\ell}}}}\cdot   \frac{e^{\nicefrac{u^{(i)}}{\mu_\ell}}}{\displaystyle
			\sum_{i\in N_\ell} e^{\nicefrac{u^{(i)}}{\mu_\ell}}}.
		\]
	\end{exm}
\subsection{Algorithmic Aspects of ARUM and Online Optimization}
Recently, discrete choice models have been linked to convex optimization frameworks \citep{prox}. Specifically, the authors incorporate prox-functions—derived from the convex conjugates of  the surplus functions—into dual averaging schemes. Building on this connection, \citet{melo2021learning} applied these insights to develop online optimization algorithms grounded in in choice surplus functions.

For completeness, we review the framework of online optimization in the $n$-experts setting. Let $A=\{1,\ldots,n\}$ be a finite set of alternatives and  let $\Delta_n$ denote the $n$-dimensional simplex over the set  $A$. Let $T\geq 2$ denote the (exogenous) number of periods. At each iteration $t$ an agent or learner observes a vector of  rewards $u_t \in \mathcal{U}\subseteq\Rn$, which is revealed after the agent made a decision $x_t \in \Delta_n$ for the $t$-th iteration. The realizations of the vector $u_t$ are determined by the environment (nature), which in principle can be adversarial. In particular, we note that no distributional assumption concerning the generated rewards is made.  Throughout the paper, we assume that the realizations of the vector $u_t$ lie in a convex bounded set given by
$\mathcal{U}\triangleq \left\{u\in \Rn : \|u\|_\infty \leq K\right\}$.
\smallskip

After committing to $x_t$, the decision-maker observes the realization of the payoff vector $u_t$ and receives an expected payoff of $\langle u_t, x_t \rangle$. The decision-maker’s objective is to select a sequence of actions $x_1, \ldots, x_T$ that minimizes regret, defined as the difference between the cumulative payoff of the best fixed decision in hindsight and the cumulative payoff actually obtained (\cite{Hahn1957}). Importantly, since no distributional assumptions are made about how the reward vectors $u_t$ are generated, the resulting regret analysis yields robust, worst-case performance guarantees. Consequently, the online decision-making process can be framed as a repeated two-player game between the decision-maker and an environment that may behave adversarially.

Let $U_t = \sum_{h=1}^{t}u_h$ denote the vector of cumulative rewards up to period $t$. Then, the online decision-making process can be described as:\\
For $t=1, \ldots, T$:
\begin{itemize}
 \item Using  $U_{t-1}$, the decision-maker chooses $x_t \in \Delta_n$;
 \item Adversary reveals $u_t \in \mathcal{Y}$;
 \item Agent gains $\la x_t, u_t \ra$.
\end{itemize}

As described earlier, the quality of online decision-making is evaluated using the notion of regret, which is formalized in the following definition.
\begin{dfn}\label{regret_dfn} Consider a time horizon of $T$ periods and a sequence of choices $x_1, \ldots, x_T$, where each $x_t \in \Delta_n$. The regret associated with this sequence  is defined as:
     \begin{equation}\label{eq:regret}
 R_T = \max_{x \in \Delta_n} \; \la x, U_T \ra - \sum_{t=1}^{T}\la x_t, u_t \ra.
 \end{equation}
\end{dfn}

 We note that definition \ref{regret_dfn}  assumes that the decision-maker has access to the entire sequence of payoff vectors $u_1, \ldots, u_T$ before making any decisions. In contrast, within the online decision-making framework, the decision-maker selects actions sequentially, relying only on information observed up to the previous period. As such, the sequence $x_1,\ldots,x_T$ captures the decision-maker's learning dynamics over time and can be naturally associated with a specific learning algorithm.

 Algorithms for online optimization can generally be divided into two main classes: Follow the Regularized Leader (FTRL) and Follow the Perturbed Leader (FTPL); see, e.g., \citet{abernethy2016perturbation}. The FTRL class relies on regularization techniques well established in optimization theory, and its regret analysis draws heavily on tools from convex analysis. In contrast, FTPL algorithms are based on the idea of perturbing the cumulative gain vector with a random variable—an approach that traces back to the seminal work of \citet{Hahn1957}.

  \citet{abernethy2016perturbation} show that the decision variable of all algorithms of these classes can be characterized by the gradient of a scalar-valued convex potential function.  \citet{melo2021learning} proves that the surplus function of many GEV models lead to algorithms where the  regret is growing by the order $\mathcal{O}(\sqrt{T})$. 
 Thus, the average regret vanishes (as $T$ grows), which is known as Hannan consistency, see e.\,g. \citet{cesa2006prediction}.  The key aspect to create GBPA  from  discrete choice  models is the convex perspective of the surplus function \eqref{consumer surplus general}:
\begin{equation}\label{eq:perspective}
\tilde{E}(U;\eta) := \eta \cdot E\left(U/\eta\right), \quad \eta > 0.
\end{equation}

Rewriting previous  equation yields:
\[
\eta \cdot E\left(U/\eta\right) = \eta \cdot \E\left[\max_{i\in A} \{U^{(i)}/\eta + \epsilon^{(i)}\}\right] = \E\left[\max_{i\in A} \{U^{(i)} + \eta \cdot \epsilon^{(i)}\}\right],
\]
which is due to Assumption \ref{ass:joint density} a stochastic smoothing of the  $\max$-function defined by \citet{abernethy2016perturbation}. Such a surplus function serves as a potential function. Hence, it remains to identify ARUMs for which the corresponding algorithms are Hannan-consistent.

In the case of full feedback (complete information), our algorithms and results are closely related to those of \citet{melo2021learning}. We revisit this result for two main reasons. First, we consider more general discrete choice models. Second, because we aim to examine the computational aspects of different algorithms under both complete and incomplete feedback settings, providing a derivation of the regret bounds helps clarify the overall exposition.

In our analysis of general ARUMs, we make use of the finite modes condition introduced by \citet{prox}.\\
\begin{dfn}
Let $g_{k,m}$ denote the density function of the difference $\epsilon^{(m)} - \epsilon^{(k)}$, where $k \neq m$. Any point $\bar{z}_{k,m} \in \mathbb{R}$ that maximizes $g_{k,m}$ is called a mode of the random variable $\epsilon^{(m)} - \epsilon^{(k)}$.
\end{dfn}
We restrict our analysis to ARUMs that satisfy the following condition.

\begin{ass}\label{ass:agents_behavior}
For all $k \neq m$, the differences of the random errors $\epsilon^{(k)} - \epsilon^{(m)}$ have  finite  modes.
\end{ass}

Let us state the blueprint for a GBPA based on random utility models satisfying Assumptions \ref{ass:joint density} and \ref{ass:agents_behavior}:
\begin{tcolorbox}
\begin{algo}[RUM-Algorithm for $n$-experts] {\ \\}\label{algo:olo}
 {\bfseries Input}: Surplus function $E$, a set of parameters $\Theta$, stepsize $\eta >0$\\
 {\bfseries Initalize}: $U_0 = \mathbf{0}$, $x_0 = \frac{1}{n}\cdot e$\\
{\bfseries For $t=1, \ldots, T$ do}:
 \begin{itemize}
     \item Choose $x_{t} =\nabla \tilde{E}(U_{t-1};\eta)$
     \item Observe $u_{t} \in \mathcal{Y}$
     \item Receive reward $\la u_t, x_t  \ra$
     \item Update $U_t = U_{t-1} + u_t$.
 \end{itemize}

\end{algo}
\end{tcolorbox}
We now show that the previous complete feedback algorithm is Hannan-consistent.

\begin{thm}\label{thm:regret_olo}
Assume that the expectation of the maximum of the random errors is bounded above, i.e., $\mathbb{E}\left[ \max_{i \in A} \{\epsilon^{(i)}\} \right] \leq \alpha$. Then Algorithm~\ref{algo:olo} is Hannan-consistent, i.e.,
\[
R_T\leq  \eta\cdot \alpha + \frac{L\cdot K^2 \cdot T}{\eta}
\]
where $L = 2 \sum_{i=1}^{n} \sum_{j \neq i} g_{i,j}(\bar{z}_{i,j})$. Optimizing the scaling parameter $\eta$ yields:
\[
R_T \leq 2\cdot \sqrt{\alpha\cdot LT}\cdot K.
\]
\end{thm}
\begin{proof}
   Due to Assumption \ref{ass:joint density}, it holds that $x_t \in \mbox{rint} (\Delta_n)$ for all $t$. Consequently, Algorithm \ref{algo:olo} is an instance of the GPBA \citep{abernethy2016perturbation}. Furthermore, under Assumption \ref{ass:agents_behavior}  the surplus function is $L$-strongly smooth w.r.t. $\|\cdot \|_\infty$ \citep{prox}. Thus, the perspective $\tilde{E}$ is $\frac{L}{\eta}$-strongly smooth and the Bregman Divergence  between $U$ and $U+u$ is bounded above, i.\,e.
   \[
   \tilde{E}(U+u;\eta) - \tilde{E}(U;\eta) - \la \nabla \tilde{E}(U;\eta),u\ra \leq \frac{L}{2\eta}\cdot \|u\|^2_\infty.
   \] 
   Therefore, by applying Theorem 1.9 of \citet{abernethy2016perturbation}, the conclusion follows immediately.
\end{proof}
To the best of our knowledge, online optimization algorithms based on general discrete choice surplus functions have not been analyzed in the existing literature.

From a mathematical standpoint, the regret bound derived in Theorem~\ref{algo:olo} is strongly influenced by the smoothness parameter of the discrete choice model, which depends on the number of alternatives.
In \citet{prox}, dimension-independent estimates of the smoothness parameter were derived for several discrete choice models. In particular, for GEV models whose generating function $G$ satisfies the following inequality for all $x = \left(x^{(1)}, \ldots, x^{(n)}\right)^T \in \mathbb{R}^n_+:$
\begin{equation}\label{eq:sc.condition.gev}
\sum_{i=1}^{n} \frac{\partial^2 G(x)}{\partial x^{(i)2}} \cdot x^{(i)2} \leq M \cdot G(x),
\end{equation}
for some constant $M \in \R$. Then, following  \citep{prox} we get that the   estimate of the smoothness parameter $L$  corresponds to:
\begin{equation}\label{eq:smoothness_gev}
   L= \frac{1}{\mu} + 2 \left(\left(1-\frac{1}{\mu}\right) + \mu M \right).
\end{equation}

Moreover, the same authors show that this condition is satisfied for the family of GNL models. This result leads to the Hannan-consistency of GNL-based online optimization algorithms, as proved in \citet{melo2021learning}. In the remainder of this section, we focus on the computational aspects of GNL-based algorithms.

Clearly, the updates of Algorithm~\ref{algo:olo} depend on the specific choice of the GNL model. The well-known exponentially weighted algorithm is based on the MNL model and therefore inherits the independence of irrelevant alternatives (IIA) property, which may be undesirable in settings where some actions exhibit correlation.

As discussed in Section~\ref{ss:ARUM}, specific instances of GNL models—such as the NL model—can incorporate complex dependence structures into the updates. At the same time, computational efficiency is maintained due to the closed-form expression provided in \eqref{eq:gev_choiceprob}. An estimate of the smoothness parameter is provided in \citet{prox}:
\begin{equation}\label{eq:smoothness_gnl}
 L_{GNL}=  \frac{2}{\min_{\ell \in L} \mu_\ell} -\frac{1}{\mu}.
\end{equation}

Let us further compare the NL  to the traditional MNL based algorithm.
For the  case of the MNL model, it follows from \eqref{eq:smoothness_gnl} that $$L_{MNL} = \frac{1}{\mu \cdot \eta}$$  while for the case of the NL\footnote{The smoothness parameter of the nested logit surplus function can be improved by the factor $\frac{1}{2}$ . This is shown by the authors of \citet{dynamic} who derive the modulus of strong smoothness.}  we get: 
$$L_{\mbox{NL}} = \frac{2}{\min_{\ell \in L} \mu_\ell \cdot \eta}.$$

It is evident that the smoothness parameter of the MNL surplus function is more favorable than that of the NL surplus function.
Let us now focus on the parameter $\alpha$. To analyze it, we can rely on the properties of the surplus function $E$, noting that 
\[
E(\mathbf{0}) = \mathbb{E}\left[\max_{i\in A} \{\epsilon^{(i)}\}\right].
\]
Due to Equation \eqref{eq:gev_surplus} we can rewrite this as 
\[
E(\mathbf{0}) = \mu \ln G\left(e^{\mathbf{0}}\right) = \mu \ln G\left(e\right).
\]
For the MNL generating function (see Example \ref{ex:mnl}) we have:
\[
G(e) =  \sum_{i =1}^{n} \left(1\right)^{\nicefrac{1}{\mu}} = n,
\]
which implies that
\begin{equation*}
E_{\mbox{MNL}}(\mathbf{0}) = \mu \cdot \ln(n).
\end{equation*}
In the case of $\eta = 1$, we have  $\alpha = \ln(n)$ which is  remarkable better than the $2 \ln(2n)$ bound derived by the moment generating function trick in \citet{abernethy2016perturbation}.

Next, let us examine the case of the NL model:
\[
G(e)= \sum_{\ell \in L} \left( \sum_{i \in N_\ell} 1^{\nicefrac{1}{\mu_\ell}} \right)^{\mu_\ell} = \sum_{\ell \in L} \left( \left|N_\ell \right|^{\mu_\ell} \right) \overset{(\star)}{\leq} \sum_{\ell \in L}   |N_\ell| = n,
\]
where in the inequality we have used  the facts that $\mu_\ell \leq 1 $ for all $\ell \in L$  and that every alternative belongs to a unique nest.
We derive a lower bound 
\begin{align*}
G(e) &= \sum_{\ell \in L} \left( \sum_{i \in N_\ell} 1^{\nicefrac{1}{\mu_\ell}} \right)^{\mu_\ell} = \sum_{\ell \in L} \left( \left|N_\ell \right|^{\mu_\ell} \right) \\ &\geq \sum_{\ell \in L} \left( \left|N_\ell \right|^{\min_{\ell \in L}\mu_\ell} \right) \overset{(\star)}{\geq} \left(\sum_{\ell \in L}  \left|N_\ell \right| \right)^{\min_{\ell \in L}\mu_\ell} \\ &\geq n^{\min_{\ell \in L}\mu_\ell}.
\end{align*}
Again, we have used the facts that $\mu_\ell \leq 1 $ for all $\ell \in L$  and that every alternative belongs to a unique nest. For inequality $(\star)$ we applied the inequality 
\[
|\sum_{i=1}^{n} x^{(i)} |^p \leq  \sum_{i=1}^{n}  |x^{(i)}|^p, \quad p \in \left(0,1\right]
\]
Altogether, this proves the following corollary:
\begin{cor}\label{cor:constants}
    For the MNL surplus function   we have  $$\alpha = \mu \cdot \ln(n).$$
    For the NL surplus function it holds:
    \[
   {\min_{\ell \in L}\mu_\ell} \cdot \ln(n) \leq \alpha \leq \ln(n).
   \]
    \end{cor}

\section{GEV Multi-armed Bandit Algorithms}\label{s3}
In this section, we address the adversarial MAB setting. Our goal is to generalize the Exp3 algorithm \citep{auer}, which is primarily based on the Gumbel distribution, to broader classes of models—specifically, the GEV and GNL models introduced in Section~\ref{ss:ARUM}. To achieve this, we show that the surplus functions characterizing  these models can be incorporated into the GBPA framework, as proposed by \citet{abernethy2016perturbation}.

\smallskip

As described in Section~\ref{s2}, in the online learning framework, the learner receives full feedback at the $t$-th round in the form of the reward vector $u_t$. This means that the decision-maker observes the reward associated with each action, regardless of which action was actually chosen. In contrast, in the MAB setting, the learner receives only limited feedback. Specifically, after selecting a probability distribution over the $n$ arms at round~$t$, a single arm $i_t$ is sampled according to this distribution, and only the reward $u_t^{(i_t)}$ of the selected arm is observed. Consequently, the agent must estimate the full reward vector based on this partial information.
\smallskip

This limitation gives rise to the well-known exploration–exploitation trade-off. Exploration involves selecting actions that may yield valuable information about uncertain rewards, whereas exploitation entails choosing the actions currently believed to offer the highest payoff. Balancing these competing objectives is a central challenge in the MAB setting and significantly increases the complexity of the learner’s task. In the adversarial MAB setting, this challenge becomes even more pronounced, as no stochastic or distributional assumptions are made about the reward sequence; see, for example, \citet{slivkins}.\\

Apart from the Exp3 algorithm, various other approaches have been developed to address the MAB problem, including Thompson Sampling and the Upper Confidence Bound (UCB) algorithm. For a comprehensive overview, see \citet{lattimore}.

In the GBPA, actions are selected based on the gradient of a convex potential function. To ensure well-defined probabilities over actions, the gradient of the (possibly smoothed) potential function, denoted $\nabla \tilde{\Phi}$, must lie in the relative interior of the $n$-dimensional probability simplex, i.e., $\nabla \tilde{\Phi} \subset \operatorname{rint}(\Delta_n)$. This condition guarantees that every action has a strictly positive probability of being chosen.
\smallskip

Let a convex potential function $\Phi$ be given, along with a sequence of reward vectors $u_1, u_2, \ldots, u_T \in [-1, 0]^n$, where each $u_t$ represents the (possibly adversarial) negative reward vector observed at round $t$. Under this setup, the GBPA provides a flexible framework for designing algorithms in the adversarial multi-armed bandit setting. Its general template is defined as follows (see \citet{abernethy2016perturbation}):
For $t=1, \ldots, T$:
\begin{itemize}
	\item Set $\hat{U}_0 = 0$;
	\item Learner samples $i_t$ according to discrete distribution $p(\hat{U}_{t-1}) = \nabla \tilde{\Phi}(\hat{U}_{t-1})$;
	\item Learner observes and gains $u^{(i_t)}_t \in \left[-1,0\right]$;
	\item Learner estimates $\hat{u}_t := \frac{u^{(i_t)}_t}{p(\hat{U}_{t-1})}\cdot e_{i_t}$; 
	\item Update $\hat{U}_t = \hat{U}_{t-1} + \hat{u}_t$.
	
\end{itemize}
Due to the inherent randomness introduced by the sampling process at each round, the performance of any algorithm is evaluated in expectation. Consequently, a well-performing algorithm is assessed in terms of its expected regret, defined as:
\begin{equation}\label{eq:exp_regret}
    \mathbb{E}\left[R_T\right] = \max_{ i\in A} U^{(i)}_T - \mathbb{E}\left[\sum_{t=1}^{T} \langle \nabla \tilde{\Phi}(\hat{U}_t), u_t \rangle  \right],
\end{equation}
where the expectation is taken over the agent's actions and the randomness in the environment.
\smallskip

The   ``loss only''  environment assumption is crucial to achieve  near-optimal (expected) regret bounds ( \citet{abernethy2016perturbation}). 
\smallskip

In the following, we introduce a new class of MAB algorithms derived from  the theory of ARUMs. Specifically, we focus on a subclass of GEV models that are differentially consistent in the sense of Definition~\ref{def:diff_cons}. This framework offers several key advantages. First, it yields a family of algorithms that are easy to implement, meaning that the sampling probabilities can be computed in closed form. Second, by selecting an appropriate ARUM specification, the learner can incorporate potential correlations among the arms in a principled way. As in previous sections, the central object of interest is the surplus function associated with the chosen GEV model.
\begin{tcolorbox}
\begin{algo}[GEV-Algorithms for multiarmed bandits] {\ \\}\label{algo:mab}
 {\bfseries Input}: Surplus function $E$ and set of parameters $\Theta$, Stepsize $\eta >0$\\
 {\bfseries Initalize}: $\hat{U}_0 = 0$\\
{\bfseries For $t=1, \ldots, T$ do}:
 \begin{itemize}
     \item Sample an arm $i_{t}$ according to the distribution $x_t=\nabla \tilde{E}(\hat{U}_{t-1};\eta)$
     \item Observe and realize reward $u^{(i_t)}_{t} \in \left[-1, 0 \, \right]^n$
     \item Estimate gain vector $\hat{u}_t = \frac{u^{(i_t)}_{t}}{x^{(i_t)}_{t}} \cdot e^{(i_t)}$
     \item Update $\hat{U}_t = \hat{U}_{t-1} + \hat{u}_t$.
 \end{itemize}

\end{algo}
\end{tcolorbox}

The structure of GEV surplus functions ensures that their gradients lie in the relative interior of the probability simplex. This property makes them well-suited for use in the GBPA framework for the MAB problem, as introduced by \citet{abernethy2016perturbation}. Algorithm~\ref{algo:mab} leverages this fact to provide concrete instances of GBPA within this setting.
\smallskip

By choosing different GEV models, one can recover a wide range of bandit algorithms, each with distinct sampling probabilities. This flexibility is captured in a unified and compact formulation through Algorithm~\ref{algo:mab}. For instance, selecting the MNL model yields the well-known Exp3 algorithm as a special case (see Example~\ref{ex:mnl}).
\smallskip

In addition to its generality, the algorithm is also computationally efficient. Since both the surplus functions and the corresponding choice probabilities in GEV models admit closed-form expressions, the implementation avoids the need for more complex techniques such as geometric resampling \citep{neu2013efficient}. This significantly simplifies both the sampling and estimation components of the algorithm.
\smallskip

The following lemma establishes a simple inequality for the expected regret
\begin{lem}\label{lem:mab_regret}
 The expected regret  of Algorithm  \ref{algo:mab} can be written as 
 \[
\mathbb{E}(R_T) \leq \mathbb{E}_{i_1, \ldots, i_T} \left[\sum_{t=1}^{T} \mathbb{E}_{i_t} \left[D_{\tilde{E}}\left(\hat{U_t},\hat{U_{t-1}}\right)| \hat{U_{t-1}} \right] 
\right] +  \tilde{E}\left(0; \eta\right)
 \]
\end{lem} 
\begin{proof}
We invoke    \citet[Lemma 1.12]{abernethy2016perturbation} and use the fact that  the convex perspective of the surplus function is a potential function. 
\end{proof}

In the previous lemma, the estimation of the reward vector $\hat{u}t$ involves an inverse scaling by the sampling probabilities $p(\hat{U}{t-1})$. As a result, the Bregman divergence $D_{\tilde{E}}(\hat{U}t, \hat{U}{t-1})$ depends on these probabilities and can, in general, become arbitrarily large. This sensitivity poses a significant issue, as it may lead to unbounded or exploding regret.

To address this, \citet{abernethy2016perturbation} introduce a regularity condition on the potential function that ensures the divergence remains bounded. This condition plays a crucial role in establishing meaningful regret guarantees within the GBPA framework.
\begin{dfn}[Differential Consistency]
	\label{def:diff_cons}
	A convex function $f$ is $C$-differentially-consistent if there exists a constant $C >0$ such that for all $U \in (-\infty, 0)^n$ and $i=1, \ldots, n$ it holds 
	\begin{equation}\label{eq:diff_cons}
	\nabla^2_{ii}f(U) \leq C\cdot \nabla_i f(U).
	\end{equation}
\end{dfn}
For $C$-differentially-consistent potential functions    an upper bound for the divergence part of Lemma \ref{lem:mab_regret}, can be proved \citep[Theorem 1.13]{abernethy2016perturbation} i.\,e. 
\begin{equation}\label{eq:bounded_breg}
\mathbb{E}_{i_t} \left[D_{\tilde{E}}\left(\hat{U_t},\hat{U}_{t-1}\right)| \hat{U}_{t-1} \right] \leq \frac{C\cdot n}{2}, \quad t=1, \ldots, T. 
\end{equation}
As discussed earlier, Algorithm~\ref{algo:mab} is capable of capturing potential dependencies among actions through the sampling process. In addition, both the sampling and estimation steps can be carried out in a numerically efficient manner, owing to the closed-form expressions available for GEV models. Given these advantages, our goal is to identify GEV models whose surplus functions satisfy $C$-differential consistency. This property is essential for ensuring bounded divergence and, consequently, favorable regret guarantees within the GBPA framework.
The following theorem characterizes GEV models satisfying this property.
\begin{thm}\label{thm:mab_gev} Consider a family of GEV models characterized by a generating function G. Suppose that for all $i = 1, \ldots, n$ and for all $x = \left(x^{(1)}, \ldots, x^{(n)}\right)^\top \in \mathbb{R}^n_+$, the following condition holds:
\begin{equation}\label{eq:diff_cons.condition.gev}
\frac{\partial^2 G(x)}{\partial (x^{(i)})^2} \cdot x^{(i)} \leq \tilde{C} \cdot \frac{\partial G(x)}{\partial x^{(i)}},
\end{equation}
for some constant $\tilde{C} \in (-1, \infty)$. Under this condition, the corresponding surplus function $E$ is $C$-differentially consistent, where $C = \tilde{C} + 1$.
\end{thm}
\begin{proof}
	To establish $C$-differential consistency, we need  to verify that the surplus function satisfies Condition~\eqref{eq:diff_cons} as stated in Definition~\ref{def:diff_cons}. 
    \smallskip
    
    To verify the condition, we need explicit expressions for the first and second derivatives of the surplus function, namely $\frac{\partial E(u)}{\partial u^{(i)}}$ and $\frac{\partial^2 E(u)}{\partial() u^{(i)})^2}$. The expression for the first derivative has been provided in Equation~\eqref{eq:gev_choiceprob}, while the second derivative has been derived by \citet{prox}. For completeness, we present both expressions below:
		\[
		\begin{array}{rcl}
		\displaystyle \frac{\partial E(u)}{\partial u^{(i)}} &=& \displaystyle\P^{(i)} =  \displaystyle \mu \frac{\partial G\left(e^{u}\right)}{\partial x^{(i)}}\cdot \frac{e^{u^{(i)}}}{G\left(e^{u}\right)}, \\ \\
		\displaystyle \frac{\partial^2 E(u)}{\partial u^{(i)2}}&=& \displaystyle  \frac{1}{\mu}\frac{\partial E(u)}{\partial u^{(i)}} \left( 1 - \frac{\partial E(u)}{\partial u^{(i)}}\right) + \left(1-\frac{1}{\mu}\right)\frac{\partial E(u)}{\partial u^{(i)}}
		+ \mu
		\frac{\partial^2 G\left(e^{u}\right)}{\partial x^{(i)2}}\cdot\frac{\left(e^{u^{(i)}}\right)^2}{G\left(e^{u}\right)}.
		\end{array}
		\] 
		We compute
		\begin{align*}
			\frac{\partial^2 E(u)}{\partial u^{(i)2}} &= \displaystyle  \frac{1}{\mu}\frac{\partial E(u)}{\partial u^{(i)}} \left( 1 - \frac{\partial E(u)}{\partial u^{(i)}}\right) + \left(1-\frac{1}{\mu}\right)\frac{\partial E(u)}{\partial u^{(i)}}
			+ \mu
			\frac{\partial^2 G\left(e^{u}\right)}{\partial x^{(i)2}}\cdot\frac{\left(e^{u^{(i)}}\right)^2}{G\left(e^{u}\right)} \\&= \frac{1}{\mu}\P^{(i)}\cdot \underbrace{\left(1-\P^{(i)}\right)}_{\leq 1} +\left(1-\frac{1}{\mu}\right)\P^{(i)} +  \mu
			\frac{\partial^2 G\left(e^{u}\right)}{\partial x^{(i)2}}\cdot\frac{\left(e^{u^{(i)}}\right)^2}{G\left(e^{u}\right)} \\ &\le \left(\frac{1}{\mu}+1-\frac{1}{\mu} \right)\P^{(i)} + \mu \cdot \frac{e^{u^{(i)}}}{G\left(e^{u}\right)} \cdot\frac{\partial^2 G\left(e^{u}\right)}{\partial x^{(i)2}}\cdot e^{u^{(i)}} \\ &\overset{\eqref{eq:diff_cons.condition.gev}}{\le} \P^{(i)} + \tilde{C}\cdot \mu \cdot \frac{\partial G\left(e^{u}\right) }{\partial x^{(i)}} \cdot \frac{e^{u^{(i)}}}{G\left(e^{u}\right)} = \left(1 +\tilde{C}\right)\cdot  \frac{\partial E(u)}{\partial u^{(i)}}.
		\end{align*}
		Altogether, we hence conclude that
		\[
			\frac{\partial^2 E(u)}{\partial u^{(i)2}} \le C\cdot  \frac{\partial E(u)}{\partial u^{(i)}},
		\]
		which shows the assertion.
\end{proof}
Theorem~\ref{thm:mab_gev} provides a sufficient condition under which GEV models are $C$-differentially consistent. A natural question arises as to how this condition relates to the strong smoothness property stated in Equation~\eqref{eq:sc.condition.gev}. Specifically, it is of interest to understand whether the smoothness condition implies, or is implied by, the differential consistency condition in Equation~\eqref{eq:diff_cons.condition.gev}, and under what circumstances the two coincide.
\begin{prop}\label{prop:1}
	Any generating function $G$ which satisifies Condition \eqref{eq:diff_cons.condition.gev} also satisfies Condition \eqref{eq:sc.condition.gev} with $M = \frac{\tilde{C}}{\mu}$.
\end{prop}
\begin{proof}
	Let us fix any $x=\left(x^{(1)}, \ldots, x^{(n)}\right)^T\in \R^n_+$ and multiply \eqref{eq:diff_cons.condition.gev} by $ x^{(i)} \in \R^n_+$ which yields for all $i=1, \ldots, n$
	\[
	\frac{\partial^2 G(x)}{\partial x^{(i)2}} \cdot x^{(i)2} \leq \tilde{C} \cdot \frac{\partial G(x)}{\partial x^{(i)}}\cdot x^{(i)}.
	\]
	Therefore, summing up over all $i=1, \ldots, n$ does not change the inequality, i.\,e.
	\begin{equation}\label{eq:euler}
	\sum_{i=1}^{n} \frac{\partial^2 G(x)}{\partial x^{(i)2}} \cdot x^{(i)2} \leq \tilde{C} \cdot \sum_{i=1}^{n} \frac{\partial G(x)}{\partial x^{(i)}}\cdot x^{(i)}.
	\end{equation}
Due to Property $G(1)$, any generating function is $\frac{1}{\mu}$- homogeneous. Applying  Euler's theorem on homogeneous functions to the right side of \eqref{eq:euler}, see for example in \citet{pemberton2011mathematics} provides
\[
 \tilde{C} \cdot \sum_{i=1}^{n} \frac{\partial G(x)}{\partial x^{(i)}} \cdot x^{(i)} =  \tilde{C}\cdot \frac{1}{\mu} G(x).
\]
Altogether, we conclude that 
\[
		\sum_{i=1}^{n} \frac{\partial^2 G(x)}{\partial x^{(i)2}} \cdot x^{(i)2} \leq \frac{\tilde{C}}{\mu} G(x).
\]
Note that $x \in \Rn_+$ has been chose arbitrarily.
\end{proof}
The previous results can be combined with those of \citet{melo2021learning} and \citet{abernethy2016perturbation} to establish a broader connection between full-feedback and bandit settings. In particular, Proposition~\ref{prop:1} shows that the class of GEV models suitable for full-feedback online optimization is at least as large as the class of GEV models applicable to bandit algorithms. This observation underscores that differential consistency is a more restrictive requirement than the conditions imposed in the full-feedback setting.\\ \vspace*{0.5 cm}

\subsection{MAB and the GNL model} The GNL model introduced by \citet{wen:2001},  provides a flexible framework for capturing correlations among choices, extending beyond the independence assumption of the standard MNL model. This structure makes  the GNL particularly appealing for MAB problems, where arms may exhibit correlated rewards. By modeling the choice probabilities through the surplus function derived from the GNL model, one can design bandit algorithms that account for nested or hierarchical relationships among actions. In this setting, the GNL-based surplus function can be incorporated into the GBPA framework, enabling efficient sampling and estimation while exploiting structural dependencies among arms.

 The generating function presented in Equation~\eqref{eq:gnl_generatingfct} characterizes the GNL model. For ease of exposition, we reproduce the generating function associated with the GNL model below:
\[
G(x)= \sum_{\ell \in L} \left( \sum_{i =1}^{n} \left(\sigma_{i\ell}\cdot x^{(i)}\right)^{\nicefrac{1}{\mu_\ell}} \right)^{\nicefrac{\mu_\ell}{\mu}}.
\]

Let us analyze the $C$-differential-consistency of GNL models.
\begin{thm}\label{cor:mab_gnl}
	For GNL the corresponding surplus function is $\frac{1}{\displaystyle\min_{\ell \in L} \mu_\ell}$- differential-consistent.
\end{thm}
\begin{proof}
	We review the following formulas, which were derived in the proof of Corollary 4 by \citet{prox}:
	 \[
	 \frac{\partial G\left(x\right)}{\partial x^{(i)}}= \frac{1}{\mu}\sum_{\ell \in L} \left( \sum_{i =1}^{n} \left(\sigma_{i\ell}\cdot x^{(i)}\right)^{\nicefrac{1}{\mu_\ell}} \right)^{\nicefrac{\mu_\ell}{\mu}-1}
	 \left(\sigma_{i\ell}\cdot x^{(i)}\right)^{\nicefrac{1}{\mu_\ell}-1} \cdot \sigma_{i\ell},
	 \]
	 and
	 \[
	 \begin{array}{rcl}
	 \displaystyle \frac{\partial^2 G(x)}{\partial x^{(i)2}} &=& \displaystyle \frac{1}{\mu}\sum_{\ell \in L} 
	 \frac{1}{\mu_\ell}\left(\frac{\mu_\ell}{\mu}-1\right)
	 \left( \sum_{i =1}^{n} \left(\sigma_{i\ell}\cdot x^{(i)}\right)^{\nicefrac{1}{\mu_\ell}} \right)^{\nicefrac{\mu_\ell}{\mu}-2}
	 \left(\left(\sigma_{i\ell}\cdot x^{(i)}\right)^{\nicefrac{1}{\mu_\ell}-1} \cdot\sigma_{i\ell}\right)^2 \\ \\
	 &+& \displaystyle \frac{1}{\mu}\sum_{\ell \in L} \left(\frac{1}{\mu_\ell}-1 \right) \left( \sum_{i =1}^{n} \left(\sigma_{i\ell}\cdot x^{(i)}\right)^{\nicefrac{1}{\mu_\ell}} \right)^{\nicefrac{\mu_\ell}{\mu}-1}
	 \left(\sigma_{i\ell}\cdot x^{(i)}\right)^{\nicefrac{1}{\mu_\ell}-2}  \cdot\sigma_{i\ell}^2.
	 \end{array}
	 \]
	 Due to $\mu_\ell \leq \mu$, $\ell \in L$, it holds:
	 \[
	 \frac{\partial^2 G(x)}{\partial x^{(i)2}} \leq \frac{1}{\mu}\sum_{\ell \in L} \left(\frac{1}{\mu_\ell}-1 \right) \left( \sum_{i =1}^{n} \left(\sigma_{i\ell}\cdot x^{(i)}\right)^{\nicefrac{1}{\mu_\ell}} \right)^{\nicefrac{\mu_\ell}{\mu}-1}
	 \left(\sigma_{i\ell}\cdot x^{(i)}\right)^{\nicefrac{1}{\mu_\ell}-2} \cdot\sigma_{i\ell}^2.
	 \]
	 We multiply by $x^{(i)}$ and get 
	 \[
	 \frac{\partial^2 G(x)}{\partial x^{(i)2}}\cdot x^{(i)} \leq \frac{1}{\mu}\sum_{\ell \in L} \left(\frac{1}{\mu_\ell}-1 \right) \left( \sum_{i =1}^{n} \left(\sigma_{i\ell}\cdot x^{(i)}\right)^{\nicefrac{1}{\mu_\ell}} \right)^{\nicefrac{\mu_\ell}{\mu}-1}
	 \left(\sigma_{i\ell}\cdot x^{(i)}\right)^{\nicefrac{1}{\mu_\ell}-1} \cdot\sigma_{i\ell}.
	 \]
	 We follow similar considerations as \citet{prox} and conclude
	  \begin{align*}
	  \frac{\partial^2 G(x)}{\partial x^{(i)2}}\cdot x^{(i)} &\leq \frac{1}{\mu}\sum_{\ell \in L} \left(\frac{1}{\mu_\ell}-1 \right) \left( \sum_{i =1}^{n} \left(\sigma_{i\ell}\cdot x^{(i)}\right)^{\nicefrac{1}{\mu_\ell}} \right)^{\nicefrac{\mu_\ell}{\mu}-1}
	  \left(\sigma_{i\ell}\cdot x^{(i)}\right)^{\nicefrac{1}{\mu_\ell}-1} \cdot\sigma_{i\ell} \\ &\leq\displaystyle\max_{\ell \in L} \left(\frac{1}{\mu_\ell}-1\right) \cdot \frac{1}{\mu}\sum_{\ell \in L}\left( \sum_{i =1}^{n} \left(\sigma_{i\ell}\cdot x^{(i)}\right)^{\nicefrac{1}{\mu_\ell}} \right)^{\nicefrac{\mu_\ell}{\mu}-1}
	  \left(\sigma_{i\ell}\cdot x^{(i)}\right)^{\nicefrac{1}{\mu_\ell}-1} \cdot\sigma_{i\ell} \\ &=  \left(\frac{1}{\displaystyle\min_{\ell \in L}\mu_\ell}-1\right) \cdot \frac{\partial G\left(x\right)}{\partial x^{(i)}}.
	  \end{align*}
	 Consequently, we can set $\tilde{C} = \left(\frac{1}{\displaystyle\min_{\ell \in L}\mu_\ell}-1\right)$. It remains to apply Theorem \ref{thm:mab_gev} which concludes the assertion by yielding $C = \frac{1}{\displaystyle\min_{\ell \in L}\mu_\ell}$. 
\end{proof}
Theorem \ref{cor:mab_gnl} enables to apply  the family of GNL models in the adverserial bandit setting.   Note that this family not only contains the multinomial logit with independent arms but also several models which are able to incorporate correlation structure such as nested logit, paired combinatorial logit \ldots . \\ In \citet{prox} the constant $M$ from Condition \eqref{eq:sc.condition.gev} is derived for GNL models, i.\,e. $M = \frac{1}{\mu}\left(\frac{1}{\displaystyle \min_{\ell \in L}\mu_\ell} -1 \right) $.  Considering Proposition \ref{prop:1} we hence see that $M=\frac{\tilde{C}}{\mu}$. Furthermore, the constant $C$ which enters the (expected) regret bound, only depends on the smallest nest parameter. Let us illustrate the constant $C$ based on the examples from Section \ref{ss:ARUM}.
\begin{rem}[Differential-Consistency of MNL and NL]
Recall the MNL model and its generating function from Example \ref{ex:mnl}.
	Note that in this example we have $\tilde{C} = \left(\frac{1}{\mu}-1\right)$ and therefore $C = \frac{1}{\mu}$. \\[0.5 em]
For the NL model from  Example \ref{ex:nl} we have
 $\tilde{C} = \left(\frac{1}{\displaystyle \min_{\ell \in L}\mu_\ell} -1 \right)$ and therefore $C = \frac{1}{\displaystyle \min_{\ell \in L}\mu_\ell}$. 
\end{rem}

We can finally state the main result of this Section.
\begin{thm}\label{thm:main}
    The Algorithm \ref{algo:mab} with a surplus function following a Generalized Nested Logit model  is at most 
    \[
\eta \cdot E(\mathbf{0}) + \frac{n \cdot T}{\displaystyle \min_{\ell \in L \cdot \eta} \mu_\ell}.
 \]
\end{thm}
\begin{proof}
    We apply Lemma \ref{lem:mab_regret} and conclude that $\tilde{E}(\mathbf{0};\eta) = \eta \cdot E(\mathbf{0})$. Furthermore, due to Theorem \ref{cor:mab_gnl} the surplus function is $\frac{1} \mu$- differentiable consistent and thus, its convex perspective is $\frac{1}{\eta\cdot \displaystyle \min_{\ell \in L}\mu_\ell}$- differentiable consistent. Together with Inequality \eqref{eq:bounded_breg} this provides an upper bound of  $\frac{n \cdot T}{\displaystyle \min_{\ell \in L \cdot \eta} \mu_\ell}$  for the divergence part, which concludes the assertion.
\end{proof}
 GNL models can not only  be used to design algorithms for online linear optimization algorithms but also for adversarial multiarmed bandit problems. This result enable the learner to design a large amount of computationally efficient algorithms with vanishing average regret and with sampling probabilities adjusted to the dependence structure of the arms.  

 \section{Generalized Gradient Bandit Algorithm}

In the preceding section, we established theoretical guarantees for a class of algorithms derived from GBPA, with regret bounds proved in adversarial settings where reward distributions can shift unpredictably.

\textcolor{black}{We now turn to the stochastic setting, showing how the GEV framework remains valuable in analyzing MAB problems under fixed reward distributions.} Our focus is the Gradient Bandit Algorithm introduced in \cite[Chapter~2.8]{sutton}, a well-regarded method in reinforcement learning that shares similarities with the EXP-3 Algorithm. Unlike EXP-3, however, its updates are driven by a preference function rather than direct rewards, and adjustments are made not only for the chosen arm but also for the unchosen ones.

We show that by allowing for a general probability distribution over arms, beyond the traditional softmax function, we enable the potential for correlated learning patterns between actions. Specifically, the distribution could follow any instance from the GEV family that ensures a closed-formed solution for the probabilities.  This broader class of probability models allows for more nuanced updates that can capture dependencies between actions, potentially leading to more efficient learning dynamics. Such an approach could enhance the algorithm's ability to adapt to complex environments by leveraging the interplay between actions, thereby improving overall performance.\smallskip

\subsection{The generalized gradient bandit algorithm} Following \cite[Chapter~2.8]{sutton}, the Gradient Bandit Algorithm at the $t$-th iteration for the $n$-Arm Bandit Problem reads as:

\begin{itemize}
	\item Sample $i_t$ according to  $\mathbb{P}(i_t = i) = \frac{e^{u_t(i)}}{\sum_{j=1}^{n} e^{u_t(j)}}$;
	\item Learner observes the rewards $R_t \in \mathbb{R}$;
	\item Update the preference \[ u_{t+1}(i)  \leftarrow u_t(i) + \alpha \cdot (R_t - \bar{R}_t) \cdot (1 - \mathbb{P}(i_t = i)) \quad \text{for } i = i_t \] \[ u_{t+1}(j) \leftarrow u_{t}(j) -  \alpha \cdot(R_t - \bar{R}_t) \cdot \mathbb{P}(i_t = j)\quad \text{for all } j \neq i_t, \]

\end{itemize}
where $\alpha$ is the stepsize parameter and $\bar{R}_t$ is the baseline helping to reduce variance \citep{sutton}. 

The reward $R_t$ is a real-valued scalar. It is determined by the realized random reward of the sampled arm in the $t$-th iteration. For clarity, we adapt the notation from the previous section and write for any iteration $t$: $$x_t^{(i)} := \mathbb{P}(i_t = i), \quad i =1, \ldots, n. $$
With this notation the preference update step can be written as 
\begin{align}\label{eq:pref_update}
\begin{split}
   u_{t+1}(i)  \leftarrow u_t(i) + \alpha \cdot (R_t - \bar{R}_t) \cdot (1 - x_t^{(i)}) \quad \text{for } i = i_t,  \\ u_{t+1}(j) \leftarrow u_{t}(j) -  \alpha \cdot(R_t - \bar{R}_t) \cdot x_t^{(j)}\quad \text{for all } j \neq i_t.
   \end{split}
\end{align}

As shown by \cite{sutton}, the Gradient Bandit Algorithm effectively performs stochastic gradient ascent on the expected reward. A key aspect is that the partial derivatives of the expected reward with respect to the preference values at time step $t$ can be expressed as the expected value of the random variable $i_t$. Precisely, it holds \citep{sutton},
\begin{equation}\label{eq:gradient_ascent}
    \frac{\partial \mathbb{E}(R_t)}{\partial u_t(i)} = \mathbb{E}_{i_t} \left[(R_t - \bar{R}_t) \cdot \frac{\partial x_t^{(i_t)}}{\partial u_t(i)} / x_t^{(i_t)} \right], \quad i =1, \ldots, n.
\end{equation}
The derivation of Equation \eqref{eq:gradient_ascent} does not depend on the concrete choice of the sampling probabilities. A requirement to derive Equation \eqref{eq:gradient_ascent} is that the sum of the partial derivatives over all arms is 0, which holds for any ARUM satisfying Assumption \ref{ass:joint density}, see for example \cite{hofbauer} or \cite{prox}. Due to Equation \eqref{eq:gradient_ascent}, the partial derivatives of the softmax probabilities determine the preference update step of the Gradient Bandit Algorithm. In facte these derivatives are given by 
\begin{align}
  \begin{split}
    x_t^{(i)} \cdot (1 - x_t^{(i)}) \quad \text{for } i = i_t, \\ 
     - x_t^{(i)} \cdot x_t^{(j)}\quad \text{for all } j \neq i_t.
  \end{split}  
\end{align}
Clearly, these derivatives treat every non-sampled arm uniformly, a direct consequence of the IIA  property inherent in the multinomial logit model. As a result, there is a uniform information update for each non-sampled arm. However, in certain scenarios, the learner may wish to exploit additional structure within the problem, such as reducing the probability of selecting certain arms by a greater magnitude compared to others. 
Generalizing to GNL opens up exciting possibilities for extending  the Gradient Bandit Algorithm, paving the way for more sophisticated and dynamic preference update strategies.  \\ 
These considerations directly lead to our Generalized Bandit Algorithm:
\begin{tcolorbox}
\begin{algo}[Generalized Gradient Bandit Algorithm for n-armed bandit] {\ \\}\label{algo:gen_bandit}
 {\bfseries Input}: GNL choice model with set of parameters $\Theta$, Stepsize $\alpha >0$\\
 {\bfseries Initalize}: $u^{(i)}_0 = 0$  for $i= 1, \ldots, n$ and $R_0 = 0, \hat{R}_0 = 0$\\
{\bfseries For $t=1, \ldots, T$ do}:
 \begin{itemize}
     \item Sample an arm $i_{t}$ according to the probabilities $x_t^{(i)}$ for $i =1, \ldots, n$, defined by the GNL choice model
     \item Observe  reward $R_{t}$ 
     \item Update preferences : \[
      u^{(i)}_{t+1} = u^{(i)}_{t} + \alpha \cdot \left[(R_t - \bar{R}_t) \cdot \frac{\partial x_t^{(i_t)}}{\partial u_t(i)} / x_t^{(i_t)} \right], \quad i= 1, \ldots\ n,
     \]
     \item Update $\bar{R}_{t+1}  = \frac{1}{t} \cdot (R_t - \bar{R}_t) $.
 \end{itemize}

\end{algo}
\end{tcolorbox}
Several observations are noteworthy. Firstly, the sampling procedure can be computed efficiently, as the choice probabilities of any GNL model are expressed in a closed form through a multiplicative two-stage process (see Section \ref{ss:ARUM}). Furthermore, the choice probabilities, as gradients of a GNL surplus function, are Lipschitz continuous \citep{prox}. Additionally, the preference update becomes more powerful by incorporating prior knowledge through parameters such as the number and partition of nests, nest parameters, and more. Lastly, the Gradient Bandit Algorithm \citep{sutton} is an instance of Algorithm \ref{algo:gen_bandit} by providing the MNL model as input. \\ \smallskip

 These considerations directly lead to our introduction of the Nested Logit Gradient Bandit Algorithm, which is an instance of the Generalized Bandit Algorithm. It is specifically designed to  provide a more refined approach to decision-making, allowing for differential treatment of non-sampled arms based on the nested structure.

\begin{tcolorbox}
\begin{algo}[Nested Logit Gradient Bandit Algorithm for n-armed bandit] {\ \\}\label{algo:nl_bandit}
 {\bfseries Input}: Partition of exclusive nests L with nest parameters $\mu_\ell$, for $\ell \in L$, Stepsize $\alpha >0$\\
 {\bfseries Initalize}: $u^{(i)}_0 = 0$  for $i= 1, \ldots, n$ and $R_0 = 0, \hat{R}_0 = 0$\\
{\bfseries For $t=1, \ldots, T$ do}:
 \begin{itemize}
     \item Sample an arm $i_{t}$ according to the probabilities
     \[ 
     x_t^{(i)} = 
		\underbrace{\frac{e^{\mu_\ell \ln \sum_{i \in N_\ell} e^{\nicefrac{u^{(i)}}{\mu_\ell}}}}{\displaystyle
			\sum_{\ell \in L} e^{\mu_\ell \ln \sum_{i \in N_\ell} e^{\nicefrac{u^{(i)}}{\mu_\ell}}}}}_{=: \hat{x}^{(\ell)}_t}\cdot   \underbrace{\frac{e^{\nicefrac{u^{(i)}}{\mu_\ell}}}{\displaystyle
			\sum_{i\in N_\ell} e^{\nicefrac{u^{(i)}}{\mu_\ell}}}}_{=: x^{(i|\ell)}_t}
		\]
     \item Observe  reward $R_{t}$ 
     \item Update preferences : 
     \begin{align*}
     &\text{for} \quad i = i_t \quad \text{and} \quad i \in N_{\ell}: \\
    &u^{(i)}_{t+1} = u^{(i)}_{t} + \alpha \cdot \left[(R_t - \bar{R}_t) \cdot \frac{1}{\mu_\ell} \cdot \left[ 1 - (1-\mu_\ell)\cdot x^{(i|\ell)}_t - \mu_\ell\cdot x^{(i)}_t \right] \right] \\[1.5 em]
     &\text{for} \quad k \neq i = i_t \quad \text{and} \quad k, i \in N_{\ell} : \\
    &u^{(k)}_{t+1} = u^{(k)}_{t} - \alpha \cdot \left[(R_t - \bar{R}_t) \cdot \frac{x^{(k)}_t}{x^{(i)}_t} \cdot \left[ x^{(i)}_t +  \frac{1-\mu_\ell}{\mu_\ell}\cdot x^{(i|\ell)}_t  \right] \right] \\[1.5 em]
    &\text{for} \quad j \neq i = i_t \quad \text{and} \quad j \in N_j \neq N_\ell: \\
     &u^{(j)}_{t+1} = u^{(j)}_{t} - \alpha \cdot(R_t - \bar{R}_t) \cdot x^{(j)}_t
     \end{align*}
     \item Update $\bar{R}_{t+1}  = \frac{1}{t} \cdot (R_t - \bar{R}_t) $.
 \end{itemize}

\end{algo}
\end{tcolorbox}
The preference update step of Algorithm \ref{algo:nl_bandit} comprises three different formulas depending on the nest structure. Notably, the update for arms in different nests follows the same formula as the Gradient Bandit Algorithm. Arms within the same nest as the sampled arm are more significantly influenced by the observed reward. Let us first analyze the preference update of the played arm $i$:
\begin{align*}
 & \frac{1}{\mu_\ell} \cdot \left[ 1 - (1-\mu_\ell)\cdot \underbrace{x^{(i|\ell)}_t}_{\leq 1} - \mu_\ell\cdot x^{(i)}_t \right]  \\ 
 \geq 
  & \frac{1}{\mu_\ell} \cdot \left[ 1 - 1+\mu_\ell - \mu_\ell\cdot x^{(i)}_t \right] \\ 
  = & 1- x^{(i)}_t
 \end{align*}
 For any non played arm $k$  in the same nest as $i$ it holds:
 \begin{align*}
      & \frac{-x^{(k)}_t}{\cdot  x^{(i)}_t} \cdot \left[ x^{(i)}_t +  \frac{1-\mu_\ell}{\mu_\ell}\cdot \underbrace{x^{(i|\ell)}_t}_{ \geq x^{(i)}_t}  \right] \\ \leq
      & -x^{(k)}_t - x^{(k)}_t  \cdot \frac{1-\mu_\ell}{\mu_\ell} \\ = 
      & -x^{(k)}_t \cdot \underbrace{\frac{1}{\mu_\ell}}_{\geq 1} \\ \leq 
      & -x^{(k)}_t 
 \end{align*}

Additionally, when the nest parameter equals one, indicating that the arms within the nest are completely uncorrelated, the update for these arms also mirrors the pattern of the Gradient Bandit Algorithm. When each arm is placed in its own nest with a parameter of $1$, the Nested Logit Gradient Bandit Algorithm becomes equivalent to the Gradient Bandit Algorithm. 
%Baseline:

%The average reward $\bar{R}$ serves as a baseline to reduce variance in updates. It helps in stabilizing the learning process.
%Traditionally, the Gradient Bandit Algorithm employs the softmax function to convert preference values into action probabilities. However, the softmax function is merely one instance within a broader class of probability distributions. By exploring this broader class, we aim to generalize the Gradient Bandit algorithm, potentially enhancing its flexibility and performance.

%This exploration is particularly compelling as it allows us to investigate the impact of alternative probability models on the algorithm's performance, offering insights into its adaptability and efficiency across diverse scenarios. The generalization of the Gradient Bandit algorithm not only enhances its theoretical appeal but also provides practical benefits, enabling customization for specific problem domains where traditional assumptions may not apply.
 \section{ Numerical Experiments: Nested Logit Bandit Algorithm}
To evaluate the practical implications of our generalization, we perform extensive numerical simulations. These simulations assess the performance of the Nested Logit Gradient Bandit Algorithm (NL Bandit Algorithm) across a variety of environments, comparing its effectiveness to the traditional approach. By systematically analyzing the outcomes, we aim to demonstrate the potential advantages of this generalized framework, providing a comprehensive understanding of its applicability and benefits.  \\ \smallskip

We recall that the sampling probabilities in Algorithm \ref{algo:nl_bandit} are expressed as the product of two probabilities, both containing exponential terms that may lead to overflow issues. To address this, we can leverage the numerically stable forms of Softmax and LogSumExp to mitigate potential overflow problems.\\ 
For clarity, our experiments are based on the framework outlined in \cite[Chapter~2.3]{sutton} and \cite[Chapter~2.8]{sutton}, hence each experiment consists of 2000 randomly generated $n$-armed bandit problems and their performance over $1000$ iterations. \\ 
\smallskip

The first environment  consists of 10 arms, i.\,e. $n=10$. Each arms' mean reward is sampled from a normal distribution with mean 4 and variance 1. which we refer to as MNL environment. We compare the performance of the classical Gradient Bandit Algorithm with that of the Nested Logit Gradient Bandit. As demonstrated in the preceding section, the classical Gradient Bandit Algorithm can be considered a special case of our Generalized Gradient Bandit Algorithm. Therefore, we refer to it as the MNL Bandit Algorithm. Our comparison for this environment includes the MNL Algorithm as well as the MNL specification of the Nested Logit Gradient Bandit Algorithm. More precisely, there is only one nest with nest parameter $1$. Additionally, we include a NL Bandit Algorithm with $5$ nests, where nest $i$ contains arms $2*i, 2*i -1$, for $i = 1, \ldots, 5$. The nest parameters are set to $0.8$. This specification is called NL $1$. The variant $NL 2$ consists of the $4$ nests ($(1,2), (3,4,5), (6,7), (8,9,10)$) with corresponding nest parameters $(0.3,0.45,0.3,0.45)$. Note that this is a rather random environment where no prior knowledge concerning the structure could be exploited.\\ 
Figure \ref{fig:env_mnl} shows the results of this first simulation. 

\begin{figure}[H]
    \centering
    \begin{subfigure}{\linewidth}
        \centering
        \includegraphics[height=0.3\textheight]{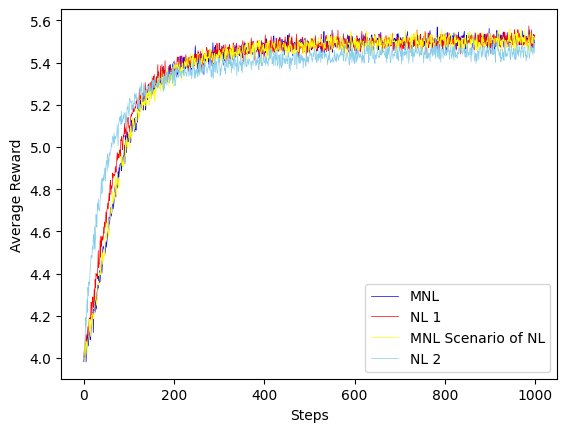} % Replace with your plot file
        %\caption{Plot 1}
        \label{fig:env_mnl}
    \end{subfigure}
    
    \vspace{0.5cm} % Adjust the space between plots if needed
    
    \begin{subfigure}{\linewidth}
        \centering
        \includegraphics[height=0.3\textheight]{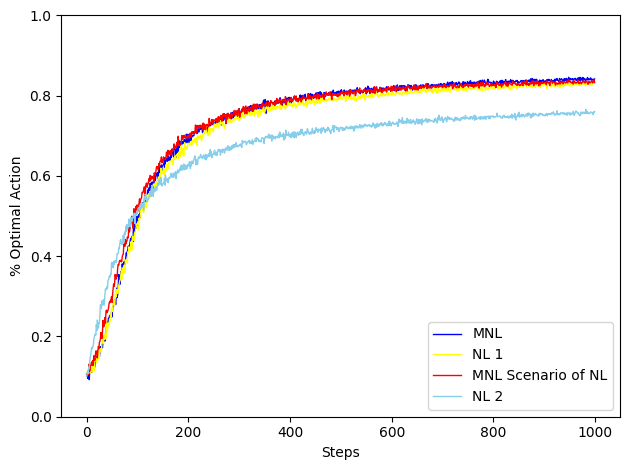} % Replace with your plot file
        %\caption{Plot 2}
        \label{fig:env_mnl}
    \end{subfigure}
    
    \caption{MNL environment}
    \label{fig:env_mnl}
\end{figure}
The MNL and the MNL-specific variant of the NL Bandit Algorithm are indistinguishable, as expected. Furthermore, the NL $1$  variant performs almost as well as the MNL, whereas the NL $2$  variant selects the best arm significantly less frequently than the others. These findings are not surprising, given that the NL $2$  model attempts to impose too much structure on the random environment. Note that the similarity of arms within each nest is set rather high in the NL $2$ variant. Consequently, we will switch to  a different environment which has more structure.\\

For the second simulation, we created a slightly different environment called the NL Environment. In this setup, $n = 9$ and three of these arms are considered better options. In particular, the mean rewards for arms $1,2$ and $3$ are sampled from a normal distribution with a mean of $7.5$, whereas the mean rewards for the other arms are sampled from a normal distribution with a mean of $2.5$. With this setup we can incorporate some meaningful structure in the nest partition. Specifically, there are three nests. Each nest consists of one better alternative and two worse alternatives. We compare the MNL Algorithm with three versions of the NL Bandit Algorithm:
\begin{itemize}
    \item NL $1$ with the same nest parameters as $0.25$, i.,e. $\mu_{\ell} = 0.25$ for $\ell = 1,2,3$.
    \item NL $2$ with the same nest parameters as $0.7$, i.,e. $\mu_{\ell} = 0.7$ for $\ell = 1,2,3$.
    \item NL $3$ with the same nest parameters as $0.45$, i.,e. $\mu_{\ell} = 0.45$ for $\ell = 1,2,3$.
\end{itemize}
    
Figure \ref{fig:env_nl} summarizes the results. Clearly, the NL Bandit Algorithm outperforms the MNL Bandit Algorithm. All three NL versions perform better in terms of average reward and the proportion of times the best arm is played, with NL $3$ being the superior variant. Moreover, the NL Bandit Algorithm is able to gain rewards earlier than the MNL Algorithm. Having a well-structured nest improves performance. Furthermore, the results suggest that the nest parameter should reflect some similarity among the arms but remain moderate enough to allow exploration of the best option within a nest.
\begin{figure}[H]
    \centering
    \begin{subfigure}{\linewidth}
        \centering
        \includegraphics[height=0.3\textheight]{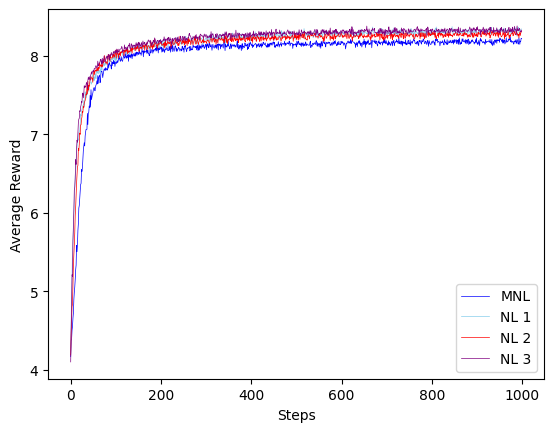} % Replace with your plot file
        %\caption{Plot 1}
        \label{fig:env_nl}
    \end{subfigure}
    
    \vspace{0.5cm} % Adjust the space between plots if needed
    
    \begin{subfigure}{\linewidth}
        \centering
        \includegraphics[height=0.3\textheight]{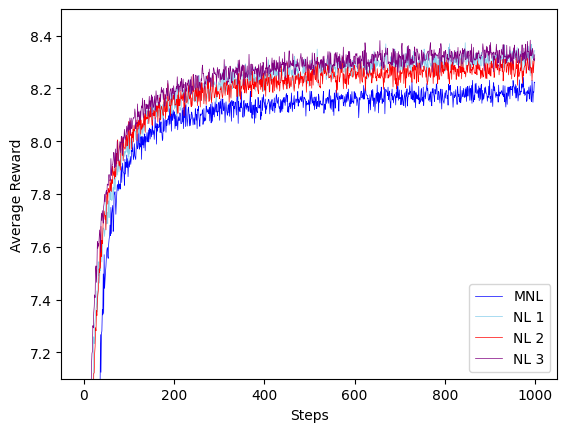} % Replace with your plot file
        %\caption{Plot 2}
        \label{fig:env_nl}
    \end{subfigure}
     
    \vspace{0.5cm} % Adjust the space between plots if needed
    
    \begin{subfigure}{\linewidth}
        \centering
        \includegraphics[height=0.3\textheight]{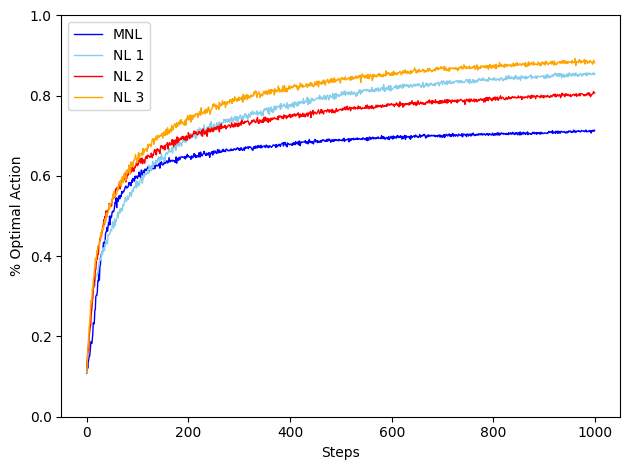} % Replace with your plot file
        %\caption{Plot 2}
        \label{fig:env_nl}
    \end{subfigure}
    
    \caption{NL environment}
    \label{fig:env_nl}
\end{figure}
To gain further insights, let's examine the learned rewards in the NL Environment. Therefore, we conduct two separate single bandit arm experiments, one with $1000$ iterations and another with $2000$ iterations, and compare how the MNL learned mean rewards stack up against those of NL $3$.
\begin{figure}[H]
    \centering
    \begin{subfigure}{\linewidth}
        \centering
        \includegraphics[height=0.3\textheight]{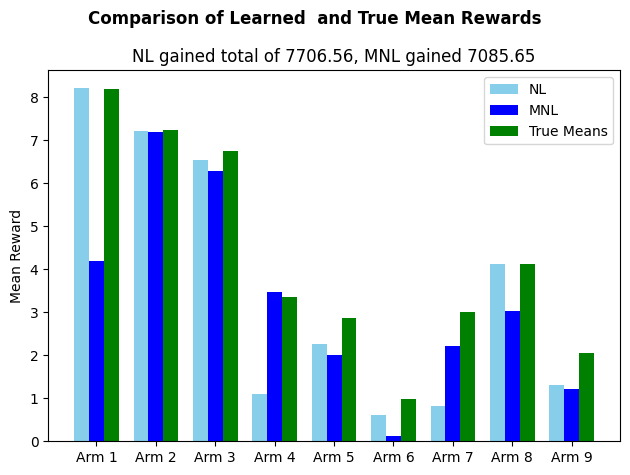} % Replace with your plot file
        %\caption{Plot 1}
        \label{fig:env_nl_learned}
    \end{subfigure}
    
    \vspace{0.5cm} % Adjust the space between plots if needed
    
    \begin{subfigure}{\linewidth}
        \centering
        \includegraphics[height=0.3\textheight]{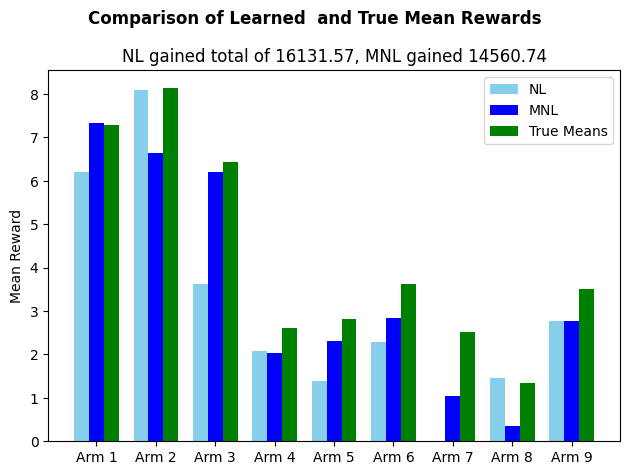} % Replace with your plot file
        %\caption{Plot 2}
        \label{fig:env_nl_learned}
    \end{subfigure}
     
    \vspace{0.5cm} % Adjust the space between plots if needed
      \caption{NL environment learned average rewards}
    \label{fig:env_nl_learned}
\end{figure}
In both simulations, the NL Bandit Algorithm accurately learns the average reward of the best arm, whereas the MNL Bandit Algorithm  fails to identify the best arm. Here, the structure proves to be very helpful. In the first plot, arm $1$, which has the highest mean reward, is in a nest with arms $4$ and $5$. The NL-based Algorithm neglects arm $4$, allowing it to learn about and exploit arm $1$ more effectively and quickly. Similar findings can be deduced from the second plot in Figure \ref{fig:env_nl_learned}. The NL Bandit Algorithm sacrifices exploring arm $7$ to exploit arm $2$ within the same nest. Furthermore, in both trials, the NL Bandit Algorithm collects significantly higher rewards. By incorporating prior knowledge about the structure of the problem, the NL Bandit Algorithm enhances the exploration/exploitation trade-off. \\
\smallskip

The final numerical tests are conducted in an environment with 25 arms. Structure is imposed by sampling the mean rewards of 24 arms from a normal distribution with a mean of 2.5, while the first arm is always considered the best option by adding 2 to the maximum sampled mean reward of the other arms. We refer to this environment as NL Large environment. For all the NL Bandit Algorithms, the best option is located in the first nest, which includes four other arms. Additionally, there are two nests, each containing ten arms. Once again, we compare the MNL Algorithm with three versions of the NL Bandit Algorithm:
\begin{itemize}
    \item NL $1$ with  nest parameters $\mu_{1} = 0.95, mu_{2} = 0.35, mu_{3} = 0.35$.
    \item NL $2$ with  nest parameters $\mu_{1} = 0.95, mu_{2} = 0.25, mu_{3} = 0.25$.
    \item NL $3$ with  nest parameters $\mu_{1} = 0.65, mu_{2} = 0.2, mu_{3} = 0.2$.
\end{itemize}
\begin{figure}[H]
    \centering
    \begin{subfigure}{\linewidth}
        \centering
        \includegraphics[height=0.3\textheight]{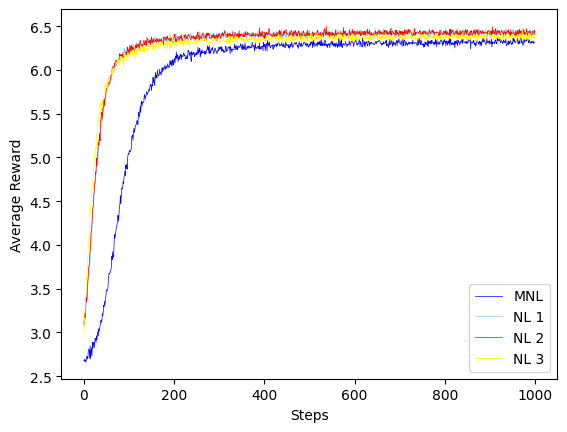} % Replace with your plot file
        %\caption{Plot 1}
        \label{fig:env_nl_large}
    \end{subfigure}
    
    \vspace{0.5cm} % Adjust the space between plots if needed
    
    \begin{subfigure}{\linewidth}
        \centering
        \includegraphics[height=0.3\textheight]{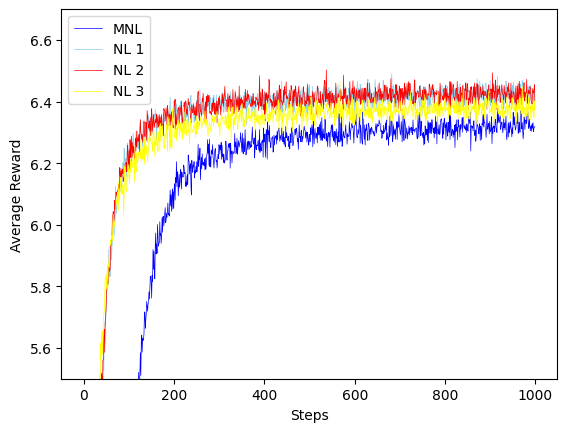} % Replace with your plot file
        %\caption{Plot 2}
        \label{fig:env_nl_large}
    \end{subfigure}
     
    \vspace{0.5cm} % Adjust the space between plots if needed
    
    \begin{subfigure}{\linewidth}
        \centering
        \includegraphics[height=0.3\textheight]{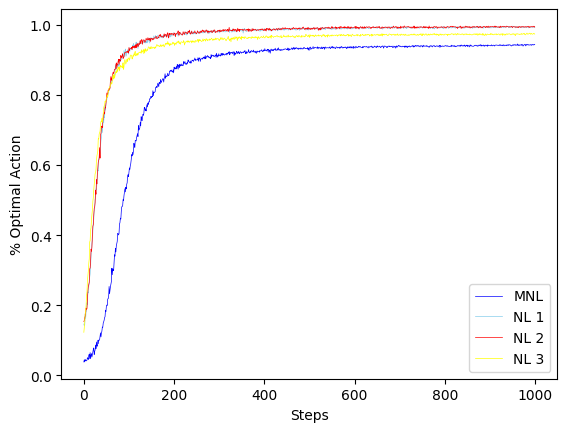} % Replace with your plot file
        %\caption{Plot 2}
        \label{fig:env_nl_large}
    \end{subfigure}
    
    \caption{NL Large environment}
    \label{fig:env_nl}
\end{figure}
\bibliographystyle{plainnat}
\bibliography{lit}

\begin{thebibliography}{28}
\providecommand{\natexlab}[1]{#1}
\providecommand{\url}[1]{\texttt{#1}}
\expandafter\ifx\csname urlstyle\endcsname\relax
  \providecommand{\doi}[1]{doi: #1}\else
  \providecommand{\doi}{doi: \begingroup \urlstyle{rm}\Url}\fi

\bibitem[Abernethy et~al.(2016)Abernethy, Lee, and
  Tewari]{abernethy2016perturbation}
Jacob Abernethy, Chansoo Lee, and Ambuj Tewari.
\newblock Perturbation techniques in online learning and optimization.
\newblock \emph{Perturbations, Optimization, and Statistics}, 233, 2016.

\bibitem[Agrawal and Goyal(2013)]{agrawal2012thompson}
Shipra Agrawal and Navin Goyal.
\newblock Thompson sampling for contextual bandits with linear payoffs.
\newblock \emph{International Conference on Machine Learning}, 28:\penalty0
  127--135, 2013.

\bibitem[Anderson et~al.(1992)Anderson, De~Palma, and Thisse]{palma}
Simon~P Anderson, Andre De~Palma, and Jacques-Francois Thisse.
\newblock \emph{Discrete choice theory of product differentiation}.
\newblock MIT press, 1992.

\bibitem[Auer et~al.(2002)Auer, Cesa-Bianchi, Freund, and Schapire]{auer}
Peter Auer, Nicolo Cesa-Bianchi, Yoav Freund, and Robert~E Schapire.
\newblock The nonstochastic multiarmed bandit problem.
\newblock \emph{SIAM journal on computing}, 32\penalty0 (1):\penalty0 48--77,
  2002.

\bibitem[Barto(2021)]{sutton}
Andrew~G Barto.
\newblock \emph{Reinforcement learning: An introduction. by richard’s
  sutton}, volume~6.
\newblock SIAM, 2021.

\bibitem[Bertsimas and Mersereau(2007)]{bertsimas2007dynamic}
Dimitris Bertsimas and Adam~J Mersereau.
\newblock A learning approach to the optimal portfolio selection problem.
\newblock \emph{Mathematics of Operations Research}, 32\penalty0 (1):\penalty0
  193--215, 2007.

\bibitem[Besbes and Zeevi(2009)]{besbes2009dynamic}
Omar Besbes and Assaf Zeevi.
\newblock Dynamic pricing without knowing the demand function: Risk bounds and
  near-optimal algorithms.
\newblock \emph{Operations Research}, 57\penalty0 (6):\penalty0 1407--1420,
  2009.

\bibitem[Cesa-Bianchi and Lugosi(2006)]{cesa2006prediction}
Nicolo Cesa-Bianchi and G{\'a}bor Lugosi.
\newblock \emph{Prediction, learning, and games}.
\newblock Cambridge university press, 2006.

\bibitem[Gatti et~al.(2012)Gatti, Lazaric, and Restelli]{gatti2012exploration}
Nicola Gatti, Alessandro Lazaric, and Marcello Restelli.
\newblock Exploration vs. exploitation in the dynamic pricing problem.
\newblock \emph{Artificial Intelligence}, 195:\penalty0 317--345, 2012.

\bibitem[Hannan(1957)]{Hahn1957}
J.~Hannan.
\newblock Approximation to \mbox{Bayes} risk in repeated play.
\newblock \emph{Contributions to the Theory of Games}, 3:\penalty0 97--139,
  1957.

\bibitem[Hofbauer and Sandholm(2002)]{hofbauer}
Josef Hofbauer and William~H Sandholm.
\newblock On the global convergence of stochastic fictitious play.
\newblock \emph{Econometrica}, 70\penalty0 (6):\penalty0 2265--2294, 2002.

\bibitem[Lattimore and Szepesv{\'a}ri(2020)]{lattimore}
Tor Lattimore and Csaba Szepesv{\'a}ri.
\newblock \emph{Bandit algorithms}.
\newblock Cambridge University Press, 2020.

\bibitem[Lee et~al.(2025)Lee, Honda, Ito, and hwan Oh]{lee2025revFTPL}
Jongyeong Lee, Junya Honda, Shinji Ito, and Min hwan Oh.
\newblock Revisiting follow-the-perturbed-leader with unbounded perturbations
  in bandit problems, 2025.
\newblock URL \url{https://arxiv.org/abs/2508.18604}.

\bibitem[Li et~al.(2010)Li, Chu, Langford, and Schapire]{li2010contextual}
Lihong Li, Wei Chu, John Langford, and Robert~E Schapire.
\newblock A contextual-bandit approach to personalized news article
  recommendation.
\newblock In \emph{Proceedings of the 19th International Conference on World
  Wide Web}, pages 661--670. ACM, 2010.

\bibitem[Li et~al.(2024)Li, Kuhn, and Ta{\c{s}}kesen]{li2024optimism}
Mengmeng Li, Daniel Kuhn, and Bahar Ta{\c{s}}kesen.
\newblock Optimism in the face of ambiguity principle for multi-armed bandits.
\newblock \emph{arXiv preprint arXiv:2409.20440}, 2024.

\bibitem[Martin et~al.(2022)Martin, Mertikopoulos, Rahier, and
  Zenati]{martin2022nested}
Matthieu Martin, Panayotis Mertikopoulos, Thibaud Rahier, and Houssam Zenati.
\newblock Nested bandits.
\newblock In \emph{International Conference on Machine Learning}, pages
  15093--15121. PMLR, 2022.

\bibitem[McFadden(1978)]{gev}
D.~McFadden.
\newblock Modeling the choice of residential location.
\newblock \emph{Transportation Research Record}, \penalty0 (673):\penalty0
  72--77, 1978.

\bibitem[McFadden(1981)]{mcfadden}
D.~McFadden.
\newblock Econometric models of probabilistic choice.
\newblock \emph{Structural analysis of discrete data with econometric
  applications}, 198272, 1981.

\bibitem[McFadden et~al.(1978)]{mcfadden1978modelling}
Daniel McFadden et~al.
\newblock Modelling the choice of residential location.
\newblock 1978.

\bibitem[Melo(Forthcoming)]{melo2021learning}
Emerson Melo.
\newblock Learning in random utility models via online decision problems.
\newblock \emph{International Journal of Economic Theory}, Forthcoming.

\bibitem[Müller et~al.(2021{\natexlab{a}})Müller, Nesterov, and
  Shikhman]{dynamic}
David Müller, Yurii Nesterov, and Vladimir Shikhman.
\newblock Dynamic pricing under nested logit demand.
\newblock \emph{Journal of Pure and Applied Functional Analysis}, 6\penalty0
  (6):\penalty0 1435--1451, 2021{\natexlab{a}}.

\bibitem[Müller et~al.(2021{\natexlab{b}})Müller, Nesterov, and
  Shikhman]{prox}
David Müller, Yurii Nesterov, and Vladimir Shikhman.
\newblock Discrete choice prox-functions on the simplex.
\newblock \emph{Mathematics of Operations Research}, 2021{\natexlab{b}}.
\newblock \doi{https://doi.org/10.1287/moor.2021.1136}.

\bibitem[Neu and Bart{\'o}k(2013)]{neu2013efficient}
Gergely Neu and G{\'a}bor Bart{\'o}k.
\newblock An efficient algorithm for learning with semi-bandit feedback.
\newblock In \emph{International Conference on Algorithmic Learning Theory},
  pages 234--248. Springer, 2013.

\bibitem[Pemberton and Rau(2015)]{pemberton2011mathematics}
Malcolm Pemberton and Nicholas Rau.
\newblock \emph{Mathematics for economists: an introductory textbook}.
\newblock Manchester University Press, 2015.

\bibitem[Robbins(1952)]{robbins1952some}
Herbert Robbins.
\newblock Some aspects of the sequential design of experiments.
\newblock \emph{Bulletin of the American Mathematical Society}, 58\penalty0
  (5):\penalty0 527--535, 1952.

\bibitem[Slivkins et~al.(2019)]{slivkins}
Aleksandrs Slivkins et~al.
\newblock Introduction to multi-armed bandits.
\newblock \emph{Foundations and Trends{\textregistered} in Machine Learning},
  12\penalty0 (1-2):\penalty0 1--286, 2019.

\bibitem[Thurstone(1927)]{thurstone}
L.~Thurstone.
\newblock A law of comparative judgment.
\newblock \emph{Psychological Review}, 34\penalty0 (4):\penalty0 273, 1927.

\bibitem[Wen and Koppelman(2001)]{wen:2001}
Chieh-Hua Wen and Frank~S Koppelman.
\newblock The generalized nested logit model.
\newblock \emph{Transportation Research Part B: Methodological}, 35\penalty0
  (7):\penalty0 627--641, 2001.

\end{thebibliography}
\end{document}